\documentclass{article}

\usepackage[final,nonatbib]{neurips_2024}

\usepackage[utf8]{inputenc} 
\usepackage[T1]{fontenc}    
\usepackage{url}            
\usepackage{booktabs}       
\usepackage{amsfonts}       
\usepackage{nicefrac}       
\usepackage{microtype}      
\usepackage{xcolor}         

\usepackage{amsfonts,amsthm,physics,proof} 
\usepackage{subcaption}
\usepackage{tikz}
\usetikzlibrary{automata, positioning, arrows}
\usepackage{thm-restate}
\usepackage{algorithm}
\usepackage{algpseudocode}
\usepackage{wrapfig}
\usepackage{enumitem}
\usepackage{amssymb}
\usepackage{dsfont}

\DeclareSymbolFont{symbolsC}{U}{pxsyc}{m}{n}
\DeclareMathSymbol{\medcirc}{\mathbin}{symbolsC}{7}

\newcommand{\defi}{\stackrel{\triangle}{=}}
\newtheorem{theorem}{Theorem}
\newtheorem{lemma}[theorem]{Lemma}
\newtheorem{proposition}[theorem]{Proposition}
\newtheorem{corollary}[theorem]{Corollary}
\newtheorem{claim}[theorem]{Claim}
\theoremstyle{definition}
\newtheorem{definition}[theorem]{Definition}
\newtheorem{example}[theorem]{Example}
\newtheorem{remark}[theorem]{Remark}

\newcommand\Real{\mathbb R}

\newcommand\prob{\mathbb P}
\newcommand\ex{\mathbb E}
\newcommand\dist{\mathcal D}

\newcommand\flag{f}
\newcommand\run{\rho}
\newcommand\runssmdp{\Omega}
\newcommand\runs{\mathrm{Runs}(S,A)}
\newcommand\fruns{\mathrm{Runs}_{\mathrm{fi}}(S,A)}

\newcommand\lab\lambda
\newcommand\defeq{:=}
\renewcommand\defi\defeq

\newcommand\nat{\mathbb N}
\newcommand\trans\Delta
\newcommand\policy\pi
\newcommand\mdp{\mathcal M}
\newcommand\dra{\mathcal A}
\newcommand\prodmdp{\mdp\otimes\dra}
\newcommand\rlavg{\vfun{\mdp}{\mathcal{R}^{\textrm{avg}}}}

\newcommand\rdisc[1]{\vfun{\mdp}{\rma^{#1}}}
\newcommand\rma{\mathcal R}

\newcommand\racc{\mathcal{J}^\mdp_{\dra}}

\newcommand\lfun\lambda 
\newcommand{\vfuns}{\mathcal{J}}
\newcommand{\vfun}[2]{\vfuns^{#1}_{#2}}
\newcommand{\rfun}{\mathcal{R}}
\newcommand\alphab{\mathcal{AP}}
\newcommand\spec{\mathcal{S}}
\newcommand{\opfun}[2]{\mathcal{J}^*(#1,#2)}
\newcommand{\oppol}[2]{\Pi_{\textrm{opt}}(#1,#2)}

\DeclareMathOperator{\inft}{InfSA}
\DeclareMathOperator{\sub}{sub}
\DeclareMathOperator{\sa}{sa}

\DeclareMathOperator{\act}{Act}
\DeclareMathOperator{\infs}{InfS}
\DeclareMathOperator{\infv}{InfV}

\newif\ifdraft
\draftfalse
\usepackage{import}
\usepackage{savesym}
\savesymbol{comment}

\ifdraft
\usepackage[draft]{commenting}
\else
\usepackage[nompar]{commenting}
\fi

\declareauthor{lo}{LO}{red} 
\authorcommand{lo}{comment}

\declareauthor{dw}{DW}{blue}
\authorcommand{dw}{comment}

\declareauthor{lxb}{LXB}{olive}
\authorcommand{lxb}{comment}

\declareauthor{ar}{AR}{magenta}
\authorcommand{ar}{comment}
\makeatletter
\renewcommand{\comm@todo@mpar}[1]{}
\makeatother

\def\divider{%
  \leavevmode\leaders\hrule height 0.6ex depth \dimexpr0.4pt-0.6ex\hfill%
  \kern0pt%
}

\renewcommand{\OpenCommentBraket}{$\lceil$}
\renewcommand{\ClosedCommentBraket}{$\rfloor$}

\usepackage{hyperref}
\usepackage[capitalize]{cleveref}

\title{Reinforcement Learning with LTL and $\omega$-Regular Objectives via Optimality-Preserving Translation to Average Rewards}

\author{%
Xuan-Bach Le$^{1*}$ \quad Dominik Wagner$^{1*}$ \\
\textbf{Leon Witzman}$^1$ \quad \textbf{Alexander Rabinovich}$^2$ \quad \textbf{Luke Ong}$^1$\\[3pt]
$^1$NTU Singapore \quad $^2$Tel Aviv University\\[3pt]
\texttt{\{bach.le,dominik.wagner,luke.ong\}@ntu.edu.sg}\\
\texttt{witz0001@e.ntu.edu.sg}\qquad
\texttt{rabinoa@tauex.tau.ac.il}
}

\begin{document}

	{\maketitle
	\def\thefootnote{*}\footnotetext{These authors contributed equally to this work.}
	}

	\begin{abstract}

    Linear temporal logic (LTL) and, more generally, $\omega$-regular objectives are alternatives to the traditional discount sum and average reward objectives in reinforcement learning (RL), offering the advantage of greater comprehensibility and hence explainability. 
    In this work, we study the relationship between these objectives. 
    Our main result is that each RL problem for $\omega$-regular objectives can be reduced to a limit-average reward problem in an optimality-preserving fashion, via (finite-memory) reward machines. 
    Furthermore, we demonstrate the efficacy of this approach by showing that optimal policies for limit-average problems can be found asymptotically by solving a sequence of discount-sum problems approximately. 
    Consequently, we resolve an open problem: optimal policies for LTL and $\omega$-regular objectives can be learned asymptotically.

	\end{abstract}

	\section{Introduction}
	Reinforcement learning (RL) is a machine learning paradigm whereby an agent aims to accomplish a task in a generally unknown environment \cite{Sutton:2018}. Traditionally, tasks are specified via a scalar reward signal obtained continuously through interactions with the environment. These rewards are aggregated over entire trajectories either through averaging or by summing the exponentially decayed rewards.
However, in many applications, there are no reward signals that can naturally be extracted from the environment.
Moreover, reward signals that are supplied by the user are prone to error in that the chosen low-level rewards often fail to accurately capture high-level objectives. 
Generally, policies derived from local rewards-based specifications are hard to understand because it is difficult to express or explain their global intent.

As a remedy, it has been proposed to specify tasks using formulas in Linear Temporal Logic (LTL) \cite{Wolff:12,Perez2024, Brazdil:14, Voloshin:22, FuT:14, Shao2023,Ding:14} or $\omega$-regular languages more generally \cite{Perez2024}. 
In this framework, the aim is to maximise the probability of satisfying a logical specification.
LTL can precisely express a wide range of high-level behavioural properties such as liveness (infinitely often $P$), safety (always $P$), stability (eventually always $P$), and priority ($P$ then $Q$ then $T$).

Motivated by this, a growing body of literature study learning algorithms for RL with LTL and $\omega$-regular objectives (e.g.\ \cite{Wolff:12,FuT:14,Perez2024,Bozkurt:2019, Sadigh:2014, Hosein:2023, Hasanbeig:2020, Gao:2019}).
However, to the best of our knowledge, all of these approaches may fail to learn provably optimal policies without prior knowledge of a generally unknown parameter. 
Moreover, it is known that neither LTL nor (limit) average reward objectives are PAC (probably approximately correct) learnable \cite{Alur:2022}. 
Consequently, approximately optimal policies can only possibly be found asymptotically but not in bounded time.
\footnote{Formally, for some $\epsilon,\delta>0$ it is impossible to learn $\epsilon$-approximately optimal policies with probability $1-\delta$ in finite time.} \lo{Citation?} \dw{This is immediate from the definition of PAC learnability. What do you want a reference for?}


In this work, we pursue a different strategy: rather than solving the RL problem directly, we study \emph{optimality-preserving} translations \cite{Alur:2022} from $\omega$-regular objectives to more traditional rewards, in particular, limit-average rewards. 
This method offers a significant advantage: it enables the learning of optimal policies for $\omega$-regular objectives by solving a single more standard problem, for which we can leverage existing off-the-shelf algorithms (e.g.~\cite{Kearns:2002,FuT:14,Perez2024}).
In this way, all future advances---in both theory and practice---for these much more widely studied problems carry over directly, whilst still enjoying significantly more explainable and comprehensible specifications.
It is well-known that such a translation from LTL to discounted rewards is impossible \cite{Alur:2022}. 
Intuitively, this is because the latter cannot capture infinite horizon tasks such as reachability or safety~\cite{Alur:2022, Yang2022,Hahn:2019}. 
Hence, we instead investigate translations to limit-average rewards in this paper.


\subsection*{Contributions}
We study reinforcement learning of $\omega$-regular and LTL objectives in Markov decision processes (MDPs) with unknown probability transitions, translations to limit-average reward objectives and learning algorithms for the latter.  In detail:



\begin{enumerate}[noitemsep]
	\item  We prove a negative result (\cref{thm:r1}): in general it is not possible to translate $\omega$-regular objectives to limit average objectives in an optimality-preserving manner if rewards are memoryless \changed[lo]{(i.e.,~independent of previously performed actions, sometimes called history-free} \changed[dw]{or Markovian)}. 
	\item On the other hand, our main result (\cref{thm:main}) resolves Open Problem 1 in \cite{Alur:2022}: such an optimality-preserving translation is possible if the reward assignment may use finite memory as formalised by reward machines~\cite{Icarte:PhD,Icarte:18a}. 

	 

	\item To underpin the efficacy of our reduction approach, we provide the first convergence proof (\cref{thm:avgconv}) of an RL algorithm (\cref{alg:avg}) for average rewards. To the best of our knowledge (and as indicated by~\cite{Dewanto:2021}), this is the first proof \emph{without assumptions on the induced Markov chains}. In particular, the result applies to multichain MDPs, which our translation generally produces, with unknown probability transitions. Consequently, we also resolve Open Problem 4 of \cite{Alur:2022}: RL for $\omega$-regular and LTL objectives can be learned in the limit (\cref{thm:ltlconv}).
\end{enumerate}


\paragraph{Outline.}
We start by reviewing the problem setup in \cref{sec:prelims}. Motivated by the impossibility result for simple reward functions, we define reward machines (\cref{sub:negative-results}).
In \cref{sec:grey} we build intuition for the proof of our main result in \cref{sec:general}. 
Thereafter, we demonstrate that RL with limit-average, $\omega$-regular and LTL objectives can be learned asymptotically (\cref{sec:lim-avg}).
Finally, we review related work and conclude in \cref{sec:related-work}.

	\section{Background}
	\label{sec:prelims}


\label{sub:prelim}

Recall that a \emph{Markov Decision Process (MDP)} is a tuple $\mdp = (S,A,s_0,P)$ where $S$ is a finite set of states, $s_0\in S$ is the initial state, $A$ is the finite set of actions and  $P: S \times A \times S \rightarrow [0,1]$ is the probability transition function such that  $\sum_{s' \in S} P(s,a,s') = 1$ for every $s \in S$ and $a \in A$. MDPs may be graphically represented; see e.g.\ \cref{fig:counter}. We let $\fruns = S \times (A \times S)^*$ and $\runs = (S \times A)^\omega$ denote the set of finite runs and the set of infinite runs in $\mdp$ respectively.


	A \emph{policy}  $\policy: \fruns \rightarrow \mathcal{D}(A)$ maps finite runs to distributions over actions. We let $\Pi(S,A)$ denote the set of all such policies. A policy $\policy$ is \emph{memoryless} if  $\policy(s_0a_0\ldots s_n) = \policy(s'_0a'_0\ldots s'_m)$ for all finite runs $s_0a_0\ldots s_n$ and $s'_0a'_0\ldots s'_m$ such that $s_n=s'_m$.
	For each MDP $\mdp$ and policy $\policy$, there is a natural induced probability measure $\dist_\policy^\mdp$ on its runs.

	The desirability of policies for a given MDP $\mdp$ can be expressed as a function $\vfuns:\Pi(S,A)\to\Real$. Much of the RL literature focuses on discounted-sum $\vfun{\mdp}{\mathcal{R}^{\gamma}}$ and limit-average reward objectives $\rlavg$, which lift a reward function $\rfun:S\times A\times S\to\Real$ for single transitions to runs $\run = s_0a_0s_1a_1\ldots$ as follows:
	\begin{align*}
		\rdisc\gamma(\policy) &\defi \ex_{\run \sim \dist_\policy^{\mdp}}\left[\sum_{i=0}^{\infty}~ \gamma^i\cdot r_i\right]&
		\rlavg(\policy) &\defi \liminf_{t \rightarrow \infty}{\ex_{\run \sim \dist_\policy^{\mdp}}\left[\frac{1}{t}\cdot \sum_{i=0}^{t-1}~r_i\right]}
	\end{align*}
	where $r_i=\mathcal{R}(s_i,a_i,s_{i+1})$ and $\gamma\in(0,1)$ is the \emph{discount factor}. 

\begin{figure}
	\centering
	\begin{subfigure}{0.43\linewidth}
		\centering	
		\begin{tikzpicture}
			[shorten >=1pt,node distance=2.3cm,on grid,auto]
			\node[state, initial] (s0) {$s_0$};
			\node[state,  right of=s0] (s1) {$s_1$};
			\path[->] 
			(s0) edge[loop above] node{$a$} (s0)
			(s0) edge[bend left, above] node{$b$} (s1)
			(s1) edge[bend left, below] node{$b$} (s0);
		\end{tikzpicture}
		\caption{An MDP where all transitions occur with probability $1$, $\lfun(s_0,b,s_1)=\{p\}$ and the rest are labeled with $\emptyset$.}
		\label{fig:counter}
	\end{subfigure}
	\hfill
	\begin{subfigure}{0.5\linewidth}
		\centering
	\begin{tikzpicture}
		[shorten >=1pt,node distance=2.3cm,on grid,auto]
		\node[state, initial] (s0) {$q_0$};
		\node[state, right of=s0] (s1) {$q_1$};
		\node[state,  right of=s1] (s2) {$q_2$};
		\path[->] 
		(s0) edge[loop above] node{$\emptyset$} (s0)
		(s0) edge[above] node{$\{p\}$} (s1)
		(s1) edge[loop above] node{$\emptyset$} (s1)
		(s1) edge[above] node{$\{p\}$} (s2)
		(s2) edge[loop above] node{$*$} (s2);
	\end{tikzpicture}
	\caption{A DRA, where $F\defeq\{(\{q_1\},\emptyset)\}$, for the objective to visit the petrol station $p$ exactly once.}
	\label{fig:dra}
	\end{subfigure}
	\caption{Examples of an MDP and DRA.}
	\label{fig:running}
\end{figure}
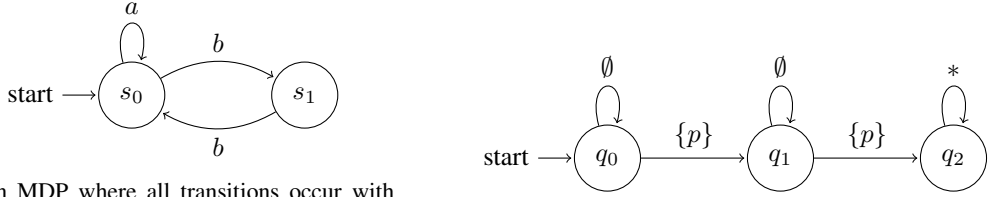

\paragraph{$\omega$-Regular Objectives.}
$\omega$-regular objectives (which subsume LTL objectives) are an alternative to these traditional objectives.
Henceforth,  we fix an alphabet $\alphab$ and a \emph{label function} $\lfun: S \times A \times S \rightarrow 2^\alphab$ for transitions, where $2^X$ is the power set of a set $X$.
Each run $\run = s_0a_0s_1a_1s_2\ldots$ induces a sequence of labels $\lfun(\run)=\lfun(s_0,a_0,s_1)\lfun(s_1,a_1,s_2)\ldots$.
Thus, for  a set $L\subseteq  (2^\alphab)^\omega$ of ``desirable'' label sequences we can consider the probability of a run's labels being in that set: $\prob_{\run\sim\dist_\policy^\mdp}[\lfun(\run)\in L]$.

\begin{example}
	\label{ex:running}
	For instance, an autonomous car may want to ``visit a petrol station exactly once'' to conserve resources~(e.g.~time or petrol). 
	Consider the MDP in~\cref{fig:counter} where the state $s_1$ represents a petrol station. We let $\alphab=\{p\}$ ($p$ for petrol),  $\lfun(s_0,b,s_1)=\{p\}$, and the rest are labeled with $\emptyset$. 
	The desirable label sequences are $L=\{\lambda_1\lambda_2\cdots\mid\text{ for exactly one }i\in\nat, \lambda_i=\{p\}\}$.	
\end{example}

In this work, we focus on $L$ which are $\omega$-regular languages. It is well known that $\omega$-regular languages are precisely the languages recognised by Deterministic Rabin Automata (DRA)~\cite{Khoussainov:2001,Kozen:2006}:

\begin{definition}\rm A DRA is a tuple $\mathcal{A} =  (Q,2^\alphab,q_0,\delta,F)$ where $Q$ is a finite state set, $2^\alphab$ is the alphabet, $q_0 \in Q$ is the initial state, $\delta: Q \times 2^\mathcal{AP} \rightarrow Q$ is the transition function, and $F = \{(A_1,R_1),\ldots,(A_n,R_n)\}$, where $A_i,R_i\subseteq Q$, is the accepting condition. Let $\run\in (2^\alphab)^\omega$ be an infinite run and $\infs(\run)$  the set of states visited infinitely often by $\run$. We say $\run$ is accepted by $\mathcal{A}$ if there exists some  $(A_i,R_i) \in F$ such that  $\run$ visits some state in $A_i$ infinitely often while visiting every states in $R_i$ finitely often, i.e. $\infs(\run) \cap A_i \neq \emptyset$ and $\infs(\run) \cap R_i = \emptyset$. 
\end{definition}
For example, the objective in \cref{ex:running} may be represented by the DRA in \cref{fig:dra}.

Thus, the desirability of  $\policy$ is the probability of $\policy$ generating an accepting sequence in the DRA $\mathcal{A}$:
\begin{align}
	\label{eq:obj}
	\vfun{\mdp}{\mathcal{A}}(\policy)~\defi~\prob_{\run\sim\dist_\policy^\mdp}[\lfun(\run)\text{ is accepted by the automaton }\dra]
\end{align}

\paragraph{Remarks.} The class of $\omega$-regular languages subsumes languages expressed by Linear Temporal Logic (LTL, see e.g.\ \cite[Ch.~5]{Baier:2008}), a logical framework in which e.g.\ reachability (eventually $P$, $\Diamond P$), safety (always $P$, $\Box P$) and reach-avoid (eventually $P$ whilst avoiding $Q$, $(\neg Q)\,\mathcal U\, P$) properties can be expressed concisely and intuitively. The specification of our running \cref{ex:running} to visit the petrol station exactly once can be expressed as the LTL formula $(\neg p)\,\mathcal U\, (p\land\medcirc\,\Box\neg p)$, where $\medcirc Q$ denotes ``$Q$ holds at the next step''. 
 Furthermore, our label function $\lfun$, which maps transitions to labels, is more general than other definitions~(e.g.~\cite{Wolff:12,FuT:14,Perez2024}) \changed[dw]{instead mapping states to labels.} \lo{Confusing to use $\lambda$ then $\lambda'$ in the preceding. Typo?} \dw{better?}
 As a result, we are able to articulate properties that involve actions, such as ``to reach the state $s$ while avoiding taking the action $a$''. 

	\paragraph{Optimality-Preserving Specification Translations.}
Rather than solving the problem of synthesising optimal policies for \cref{eq:obj} directly, we are interested in reducing it to more traditional RL problems and applying off-the-shelf RL algorithms to find optimal policies. To achieve this, the reduction needs to be \emph{optimality preserving}\footnote{This definition makes sense both for the case of reward functions $\rma_{(S,A,\lfun,\dra)}:S\times A\times S\to\Real$ and reward machines (introduced in the subsequent section).}:

\dw{What do you think of this modified definition?}
\lo{OK. But note that we have not defined reward machines yet, and the type of $\rma_{(S,A,\lfun,\dra)}$ is $S\times A\times S\to\Real$, which rules out memoryless reward machines.}
\dw{Is the clarifying footnote helpful?}
	\begin{definition}[\cite{Alur:2022}] \rm
		An \emph{optimality-preserving specification translation} from $\omega$-regular objectives to limit-average rewards is a computable function mapping each tuple $(S,A,\lfun,\dra)$ to $\rma_{(S,A,\lfun,\dra)}$ s.t. 
		\begin{quote}
			policies maximising $\rlavg$ also maximise $\racc$, where $\rma\defeq \rma_{(S,A,\lfun,\dra)}$
		\end{quote}
		for every MDP $\mdp=(S,A,s_0,P)$, label function $\lfun:S\times A\times S\to 2^{\mathcal{AP}}$ and DRA $\dra$.
	\end{definition} 
	We stress that since the probability transition function $P$ is generally not known, the specification translation may not depend on it.

	\section{Negative Result and Reward Machines}\label{sub:negative-results}
	


Reward functions emit rewards purely based on the transition being taken without being able to take the past into account, whilst DRAs have finite memory. Therefore, there cannot generally be optimality-preserving translations from $\omega$-regular objectives to limit average rewards provided by reward functions:
\begin{proposition}\label{thm:r1}
	There is an MDP $\mdp$ and an $\omega$-regular language $L$ for which it is impossible to find a reward function $\mathcal{R}:S \times A \times S \rightarrow\mathbb{R}$ such that every $\rlavg$-optimal policy of $\mdp$ also maximises the probability of membership in $L$.
\end{proposition}

  Remarkably, this rules out optimality-preserving specification translations even if transition probabilities are fully known\footnote{In \cref{app:neg} we show another negative result (\cref{prop:fun-general}): even for a strict subset of $\omega$-regular specifications such translations are impossible.}. 
  
  \begin{proof} Consider the deterministic MDP in~\cref{fig:counter} and the objective of \cref{ex:running} ``to visit $s_1$ exactly once'' expressed by the DRA $\dra$ in \cref{fig:dra}. 
	Assume towards contradiction there exists a reward function $\rma:S\times A\times S\to\Real$ such that optimal policies w.r.t.\ $\rlavg$ maximise acceptance by $\dra$.
	Note that every policy $\pi^*$ maximising acceptance by the DRA induces the run $s_0(as_0)^nbs_1bs_0(as_0)^\omega$ for some $n\in\nat$, and $\racc(\pi^*) = 1$.  
	Thus, its limit-average reward is $\rlavg(\pi^*) = \mathcal{R}(s_0,a,s_0)$. Now, consider the policy $\pi$  always selecting action $a$ with probability $1$. As the run induced by $\pi$ is $s_0(as_0)^\omega$, we deduce that $\racc(\pi) = 0$ and $\rlavg(\pi) = \mathcal{R}(s_0,a,s_0) = \rlavg(\pi^*)$, which is a contradiction since $\pi$ is not $\racc$-optimal.
  \end{proof}

	Since simple reward functions lack the expressiveness to capture $\omega$-regular objectives, we employ a generalisation, reward machines \cite{Icarte:PhD,Icarte:18a}, whereby rewards may also depend on an internal state:

	
	\begin{definition} \rm A \emph{reward machine (RM)} is a tuple $\mathcal{R} = (U,u_0,\delta_u,\delta_r)$ where $U$ is a finite set of states, $u_0 \in U$ is the initial state, $\delta_r: U  \times (S \times A \times S) \rightarrow \mathbb{R}$ is the reward function, and $\delta_u:U \times (S \times A  \times S) \rightarrow U$ is the update function.
\end{definition}	

Intuitively, a RM $\mathcal{R}$ utilises the current transition to update its states through~$\delta_u$ and assigns the rewards through~$\delta_r$.
For example,~\cref{fig:rm} depicts a reward machine for the MDP of \cref{fig:counter}, where the states count the number of visits to $s_1$ (0 times, once, more than once).

	Let $\run = s_0a_0s_1\cdots$ be an infinite run. Since $\delta_u$ is deterministic, it induces a sequence $u_0u_1\ldots$ of states in $\rma$, where $e_i=(s_i,a_i,s_{i+1})$ and $u_{i+1} = \delta_u(u_i,e_i)$.
	The \emph{limit-average reward} of a policy $\pi$ is defined as:
	\begin{align*}
	\rlavg(\pi) ~\defi~ \liminf_{t \rightarrow \infty}{\ex_{\run \sim \dist_\pi^{\mdp}}\left[\frac{1}{t} \sum_{i=0}^{t-1}~\delta_r(u_i, e_i)\right]}
	\end{align*}
	It is seen that limit-average optimal policies $\policy^*$ for the MDP in \cref{fig:counter} and the RM in \cref{fig:rm} eventually select action $b$ exactly once in state $s_0$ to achieve $\rlavg(\policy^*)=1$.

	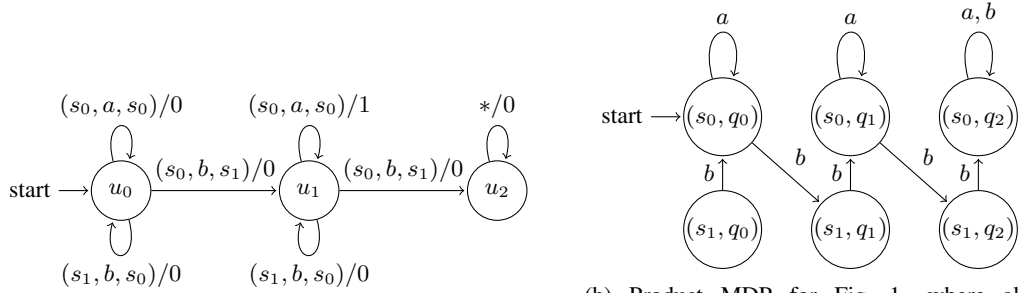
\begin{figure}
		\centering
		\begin{subfigure}{0.55\linewidth}
			\centering
			{\footnotesize
		\begin{tikzpicture}
			[shorten >=1pt,node distance=2.3cm,on grid,auto]
			\node[state, initial] (s0) {$u_0$};
			\node[state, right=25mm of s0] (s1) {$u_1$};
			\node[state,  right=25mm of s1] (s2) {$u_2$};
			\path[->] 
			(s0) edge[loop above] node{$(s_0,a,s_0)/0$} (s0)
			(s0) edge[loop below] node{$(s_1,b,s_0)/0$} (s0)
			(s0) edge[above] node{$(s_0,b,s_1)/0$} (s1)
			(s1) edge[loop above] node{$(s_0,a,s_0)/1$} (s1)
			(s1) edge[loop below] node{$(s_1,b,s_0)/0$} (s1)
			(s1) edge[above] node{$(s_0,b,s_1)/0$} (s2)
			(s2) edge[loop above] node{$*/0$} (s2);
		\end{tikzpicture}}
		\caption{A reward machine for the objective of visiting the petrol station exactly once. (The rewards are given following ``/''.)}
		\label{fig:rm}
		\end{subfigure}
		\hfill
		\begin{subfigure}{0.42\linewidth}
			\centering
			{\footnotesize
			\begin{tikzpicture}
				[shorten >=1pt,node distance=2.6cm,on grid,auto]
				\node[state, inner sep=0pt, initial] (s0q0) {$(s_0,q_0)$};
				\node[state, inner sep=0pt, right=17mm of s0q0] (s0q1) {$(s_0,q_1)$};
				\node[state, inner sep=0pt, right=17mm of s0q1] (s0q2) {$(s_0,q_2)$};
				\node[state, inner sep=0pt, below=15mm of s0q0] (s1q0) {$(s_1,q_0)$};
				\node[state, inner sep=0pt, below=15mm of s0q1] (s1q1) {$(s_1,q_1)$};
				\node[state, inner sep=0pt, below=15mm of s0q2] (s1q2) {$(s_1,q_2)$};
				\path[->] 
				(s1q0) edge node{$b$} (s0q0)
				(s0q0) edge[loop above] node{$a$} (s0q0)
				(s0q0) edge node{$b$} (s1q1)
				(s1q1) edge node{$b$} (s0q1)
				(s0q1) edge[loop above] node{$a$} (s0q1)
				(s0q1) edge node{$b$} (s1q2)
				(s1q2) edge node{$b$} (s0q2)
				(s0q2) edge[loop above] node{$a,b$} (s0q2)
				;
			\end{tikzpicture}}
			\caption{Product MDP for \cref{fig:running}, where all transitions have probability $1$ and $F_\mdp\defeq\{(\{(s_0,q_1),(s_1,q_1)\},\emptyset)\}$.}
			\label{fig:prodmdp}
		\end{subfigure}
		\caption{A reward machine and the product MDP for the running \cref{ex:running}.}
	\end{figure}

	In the following two sections, we present a general translation from $\omega$-regular languages to  limit-average reward machines, and we show that our translation is optimality-preserving (\cref{thm:main}).

\paragraph{Remarks.} Our definition of RM is more general than the one presented in~\cite{Icarte:PhD,Icarte:18a}, where $\delta'_u: U\rightarrow [S \times A \times S \rightarrow \mathbb{R}]$ and $\delta'_r:U \times 2^\alphab \rightarrow U$. Note that $(\delta'_u,\delta'_r)$ can be reduced to $(\delta_u,\delta_r)$ by expanding the state space of the RM to include the previous state and utilising the inverse label function~$\lfun^{-1}$. It is worth pointing out that \cref{thm:main} does not contradict a negative result in~\cite{Alur:2022} regarding the non-existence of an optimality-preserving translation from LTL constraints to \emph{abstract} limit-average reward machines (where only the \emph{label} of transitions is provided to $\delta_u$ and $\delta_r$).

	\section{Warm-Up: Transitions with Positive Probability are Known}


	\label{sec:grey}

To help the reader gain intuition about our construction, we first explore the situation where the support $\{(s,a,s')\in S\times A\times S\mid P(s,a,s')>0\}$ of the MDP's transition function is known. Crucially, we do not assume that the \emph{magnitude} of these (non-zero) probabilities are known. Subsequently, in \cref{sec:general}, we fully eliminate this assumption. 

This assumption allows us to draw connections between our problem and a familiar scenario in probabilistic model checking \cite[Ch.~10]{Baier:2008}, where the acceptance problem for $\omega$-regular objectives can be transformed into a reachability problem. Intuitively, our reward machine monitors the state of the DRA and provides reward $1$ if the MDP and the DRA are in certain ``good'' states ($0$ otherwise).


For the rest of this section, we fix an MDP without transition function $(S,A,s_0)$, a set of possible transitions $E\subseteq S\times A\times S$, a label function $\lfun:S\times A\times S\to 2^{\alphab}$ and a DRA $\dra=(Q,2^\mathcal{AP},q_0,\delta,F)$. Our aim is to find a reward machine $\rma$ such that for every transition function $P$ compatible with $E$ (formally: $E=\{(s,a,s')\mid P(s,a,s')>0\}$), optimal policies for limit-average rewards are also optimal for the acceptance probability of the DRA $\dra$.

	
	 \subsection{Product MDP and End Components}
	First, we form the \emph{product MDP} $\prodmdp$ (e.g.~\cite{Wolff:12,FuT:14}), which synchronises the dynamics of the MDP $\mdp$ with the DRA $\dra$. Formally, $\prodmdp = (V,A,v_0,\Delta,F_{\mdp})$ where $V = S \times Q$ is the set of states, $A$ is the set of actions, $v_0 = (s_0,q_0)$ is the initial state. The transition probability function $\Delta: V \times A \times V \rightarrow [0,1] $ satisfies $\Delta(v,a,v') = P(s,a,s')$ given that $v = (s,q)$, $v' = (s',q')$, and $\delta(q,\lfun(s,a,s')) = q'$. The accepting condition is $F_{\mdp} = \{(A_1',R'_1),(A_2',R'_2),\ldots\}$ where $A'_i = S \times A_i$, $B'_i = S \times B_i$, and $(A_i,B_i) \in F$.  
	A run $\run=(s_0,q_0)a_0\cdots$ is accepted by $\prodmdp$ if there exists some  $(A'_i,R'_i) \in F_{\mdp}$ such that  $\infv(\run) \cap A'_i \neq \emptyset$ and $\infv(\run) \cap R'_i = \emptyset$, where $\infv$ is the set of states $(s,v)$ in the product MDP visited infinitely often by $\run$.
	
Note that product MDPs have characteristics of both MDPs and DRAs which neither possesses in isolation: transitions are generally probabilistic and there is a notation of acceptance of runs.
	For example, the product MDP for \cref{fig:running} is shown in \cref{fig:prodmdp}.
	Due to the deterministic nature of the DRA $\dra$, every run $\run$ in $\mdp$ gives rise to a unique run $\run^\otimes$ in $\prodmdp$. Crucially, for every policy $\policy$,
		\begin{align}
			\label{eq:prodmdp}
			\prob_{\run\sim\dist_\pi^\mdp}[\run\text{ is accepted by }\dra]=\prob_{\run\sim\dist_\pi^\mdp}[\run^\otimes\text{ is accepted by }\prodmdp]
		\end{align}

	We make use of well-known almost-sure characterisation of accepting runs via the notion of accepting end components:
	
	\begin{definition} \rm An \emph{end component} (EC) of $\prodmdp = (V,A,v_0,\Delta,F_{\mdp})$ is   a pair  $(T,\act)$ where $T \subseteq V$ and $\act:T \rightarrow 2^{A}$ satisfies the following conditions
		\begin{enumerate}[noitemsep]
			\item For every $v \in T$ and $a \in\act(v)$, we have $\sum_{v' \in T} \Delta(v,a,v') = 1$, and
			\item The graph $(T,\rightarrow_{\act})$ is strongly connected, where $v \rightarrow_{\act} v'$ iff $\Delta(v,a,v') >~0$ for some $a\in\act(v)$.
		\end{enumerate}	
		$(T,\act)$ is an \emph{accepting EC (AEC)} if  $T \cap A'_i \neq \emptyset$ and $T \cap B'_i = \emptyset$ for some $(A'_i,B'_i) \in F_{\mdp}$.
	\end{definition}
	Intuitively, an EC is a strongly connected sub-MDP.
	For instance, for the product MDP in \cref{fig:prodmdp} there are five end components, $(\{(s_0,q_0)\},(s_0,q_0)\mapsto\{a\})$, $(\{(s_0,q_1)\},(s_0,q_1)\mapsto\{a\})$, $(\{(s_0,q_2)\},(s_0,q_2)\mapsto\{a\})$, $(\{(s_0,q_2)\},(s_0,q_2)\mapsto\{b\})$  and $(\{(s_0,q_2)\},(s_0,q_2)\mapsto\{a,b\})$. $(\{(s_0,q_1)\},(s_0,q_1)\mapsto\{a\})$ is its only accepting end component.
	
	It turns out that, almost surely, a run is accepted iff it enters an accepting end component and never leaves it \cite{Alfaro:PhD}.
	Therefore, a natural idea for a reward machine is to use its state to keep track of the state $q\in Q$ the DRA is in and give reward 1 to transitions $(s,a,s')$ if $(s,q)$ is in some AEC (and $0$ otherwise). 
Unfortunately, this approach falls short since the AEC may contain non-accepting ECs, thus assigning maximal reward to sub-optimal policies.\footnote{To illustrate this point, consider the product MDP $(\{s_0,s_1\},\{a,b\},s_0,P,F)$ where $P(s_0,b,s_0) = P(s_0,a,s_1) = P(s_1,a,s_0) = 1$ and $F=\{(\{s_1\},\emptyset)\}$, i.e. the objective is to visit~$s_1$ infinitely often.}
As a remedy, we introduce a notion of minimal AEC, and ensure that only runs eventually committing to one such minimal AEC get a limit-average reward of 1.

\begin{definition}
	An AEC $(T,\act)$ is an \emph{accepting simple EC (ASEC)} if $|\act(v)| = 1$ for every $v \in T$.
\end{definition}



Let $\mathcal C_1=(T_1,\act_1),\ldots,\mathcal C_n=(T_n,\act_n)$ be a collection of ASECs covering all states in ASECs, i.e.\ if $(s,q)$ is in some ASEC then $(s,q)\in {T_1 \cup {\cdots} \cup T_n}$. In particular, $n\leq |S\times Q|$ is sufficient.

We can prove that every AEC contains an ASEC (see \cref{lem:asec} in \cref{app:grey}). Consequently,
	\begin{restatable}{lemma}{asecsuff}
		\label{lem:asecsuff}
		Almost surely, if $\run$ is accepted by $\dra$ then $\run^\otimes$ reaches a state in some ASEC $\mathcal C_i$ of $\prodmdp$.
	\end{restatable}

	\subsection{Reward Machine and Correctness}
	\label{sec:rmc}

	Next, to ensure that runs eventually commit to one such ASEC we introduce the following notational shorthand:
	for $(s,q)\in {T_1 \cup \cdots \cup T_n}$, let $\mathcal C_{(s,q)}=(T_{(s,q)},\act_{(s,q)})$ be the $\mathcal C_i$ with minimal $i$ containing $(s,q)$, i.e. $C_{(s,q)}\defeq C_{\min\{1\leq i\leq n\mid (s,q)\in T_i\}}$. 

Intuitively, we give a reward of 1 if $(s,q)$ is in one of the $\mathcal C_1,\ldots,\mathcal C_n$. However, once an action is performed which deviates from $\act_{(s,q)}$ no rewards are given thereafter, thus resulting in a limit average reward of $0$.

A state in the reward machine has the form $q\in Q$, keeping track of the state in the DRA, or $\bot$, which is a sink state signifying that in a state in $\mathcal C_1,\ldots,\mathcal C_n$ we have previously deviated from $\act_{(s,q)}$.

Finally, we are ready to formally define the reward machine $\rma=\rma_{(S,A,\lfun,\dra)}$ exhibiting our specification translation as $(Q\cup\{\bot\},q_0,\delta_u,\delta_r)$, where
\begin{align*}
	\delta_u(u,(s,a,s'))&\defeq
	\begin{cases}
		\bot&\text{if }u=\bot\text{ or}\\
		&\left((s,u)\in T_1 \cup\cdots\cup T_n\text{ and }a\not\in \act_{(s,u)}(s,u)\right)\\
		\delta(u,\lfun(s,a,s'))&\text{otherwise}
	\end{cases}\\
	\delta_r(u,(s,a,s'))&\defeq
	\begin{cases}
		1&\text{if }u\neq\bot\text{ and }(s,u)\in T_1\cup\cdots\cup T_n\\
		0&\text{otherwise}
	\end{cases}
\end{align*}

For our running example, this construction essentially yields the reward machine in \cref{fig:rm} (with some inconsequential modifications cf.~\cref{fig:construction} in \cref{app:grey}).

\begin{theorem}
	\label{thm:grey}
	For all transition probability functions $P$ with support $E$, policies maximising the limit-average reward w.r.t.\  $\rma$ also maximise the acceptance probability of the DRA $\dra$.
\end{theorem}	
This result follows immediately from the following (the full proof is presented in \cref{app:grey}):


\begin{restatable}{lemma}{corgrey}
	\label{lem:corgrey}
	Let $P$ be a probability transition function with support $E$ and $\mdp\defeq (S,A,s_0,P)$.
	\begin{enumerate}[noitemsep,topsep=0pt]
\item For every policy $\policy$, $\rlavg(\pi)\leq\racc(\pi)$.
\item For every policy $\policy$, there exists some policy $\policy'$ satisfying
$\racc(\policy)\leq\rlavg(\policy')$.
	\end{enumerate}
\end{restatable}
\begin{proof}[Proof sketch]
1. By construction, every run receiving a limit-average reward of $1$, must have entered some ASEC $\mathcal C_i$ and never left it. Furthermore, almost surely all states are visited infinitely often and the run is accepted by definition of accepting ECs. 

2. By \cref{lem:asecsuff}, almost surely, a run is only accepted if it enters some $\mathcal C_i$. We set $\policy'$ to be the policy agreeing with $\policy$ until reaching one of the $\mathcal C_1,\ldots,\mathcal C_n$ and henceforth following the action $\act_{(s_t,q_t)}(s_t,q_t)$, where $q_t$ is the state of the DRA at step $t$, yielding a guaranteed limit-average reward of $1$ for the run by construction.
		\qedhere
\end{proof}

\changed[dw]{
\begin{remark}
		Our construction considers a collection of ASECs covering all states in ASECs. Whilst it does not necessarily require listing all possible ASECs but only (up to) one ASEC per state, it is unclear whether this can be obtained in polynomial time. In \cref{sec:eff}, we present an alternative (yet more complicated) construction which has polynomial time complexity.
\end{remark}
}

	\section{Main Result}
    \label{sec:general}

In this section, we generalise the approach of the preceding section to prove our main result:
\begin{restatable}{theorem}{main}
	\label{thm:main}
There exists an optimality-preserving translation from $\omega$-regular languages to limit-average reward machines.	
\end{restatable}
Again, we fix an MDP without transition function $(S,A,s_0)$, a label function $\lfun:S\times A\times S\to 2^{\alphab}$ and a DRA $\dra=(Q,2^\mathcal{AP},q_0,\delta,F)$.
Note that the ASECs of a product MDP are uniquely determined by the non-zero probability transitions. Thus, for each set of transitions  $E\subseteq (S\times Q)\times A\times (S\times Q)$,  
we let $\mathcal C^E_1=(T_1,\act_1),\ldots,\mathcal C^E_n=(T_n,\act_n)$ denote a collection of ASECs covering all states in ASECs w.r.t.\ the MDPs in which $E$ is the set of non-zero probability transitions.\footnote{To achieve the same number $n$ of ASECs we can add duplicates. If there are no ASECs we can set $T_i\defeq\emptyset$.} 
Then, for each set $E$ and state $(s,q)\in T^E_1\cup\cdots\cup T^E_n$, we let $\mathcal C_{(s,q)}^E=(T_{(s,q)}^E,\act_{(s,q)}^E)$ be the ASEC $\mathcal C_i^E$ that contains $(s,q)$ in which the index $i$ is minimal.

Our reward machine $\rma=\rma_{(S,A,\lfun,\dra)}$ extends the ideas from the preceding section. 
Importantly, we keep track of the set of transitions $E$ taken so far and assign rewards according to our current knowledge about the graph of the product MDP. 
Therefore, we propose employing states of the form $(q, \flag, E)$, where $q\in Q$ keeps track of the state of the DRA, $\flag\in\{\top,\bot\}$ is a \emph{status flag} and $E\subseteq (S\times Q)\times A\times (S\times Q)$ memorises the transitions in the product MDP encountered thus far.

Intuitively, we set the flag to $\bot$ if we are in MDP state $s$, $(s,q)$ is in one of the $\mathcal C_1^E,\ldots,\mathcal C_n^E$ and the chosen action deviates from $\act_{(s,q)}^E(s,q)$. We can recover from $\bot$ by discovering new transitions.
Besides, we give reward $1$ if $\flag=\top$ and $(s,q)$ is in one of the $\mathcal C_1^E,\ldots,\mathcal C_n^E$ (and $0$ otherwise).

 The status flag is required since discovering new transitions will change the structure of (accepting simple) end components. Hence, differently from the preceding section, it is not sufficient to have a single sink state.

The initial state of our reward machine is $u_0\defeq (q_0,\top,\emptyset)$ and we formally define the update and reward functions as follows:
\begin{align*}
	\delta_u((q,\flag,E),(s,a,s'))&\defeq
	\begin{cases}
		(q',\bot,E) &\text{if }\flag=\bot\text{ and } e\in E  \\
        (q',\bot,E)&\text{if } \flag=\top,e\in E, (s,q)\in T_1^E\cup\cdots\cup T_n^E\text{ and }\\
        & a\not\in \act_{(s,q)}^E(s,q)\\
        (q',\top, E\cup \{e\}) &\text{otherwise}
	\end{cases}\\
	\delta_r((q,\flag,E),(s,a,s'))&\defeq
	\begin{cases}
		1&\text{if }\flag=\top,(s,q)\in T_1^E\cup\cdots\cup T_n^E\\
		0&\text{otherwise}
	\end{cases}
\end{align*}
where $q'\defeq\delta(q,\lfun(s,a,s'))$ and $e\defeq ((q,s),a,(q',s'))$.

	\begin{example}
		For our running example (see \cref{ex:running,fig:running}) initially no transitions are known (hence no ASECs). Therefore, all transitions receive reward $0$. Once action $a$ has been performed in state $s_0$ in the MDP $\mdp$ and $(q_1,f,E)$ in the reward machine $\rma$, we have discovered the ASEC $(\{(s_0,q_1)\},(s_0,q_1)\mapsto \{a\})$ and a reward of $1$ is given henceforth unless action $b$ is selected eventually. In that case, we leave the ASEC and we will not discover further ASECs since there is only one. From here, it is not possible to return to state $q_1$ in the DRA and henceforth only reward $0$ will be obtained.
	\end{example}

\cref{thm:main} is proven by demonstrating an extension of \cref{lem:corgrey} (see \cref{app:general}):
\begin{restatable}{lemma}{corblack}
	\label{lem:corblack} 
	Suppose $\mdp=(S,A,s_0,P)$ is an arbitrary MDP.
	\begin{enumerate}[noitemsep,topsep=0pt]
\item For every policy $\policy$, $\rlavg(\pi)\leq\racc(\pi)$.
\item For every policy $\policy$, there exists some policy $\policy'$ satisfying
$\racc(\policy)\leq\rlavg(\policy')$.
	\end{enumerate}
\end{restatable}
Note that \cref{lem:corblack} immediately proves that the reduction is not only optimality preserving (\cref{thm:main}) but also robust: every $\epsilon$-approximately limit-average optimal policy is also $\epsilon$-approximately optimal w.r.t.\ $\racc$.
This observation is important because \emph{exactly} optimal policies for the limit average problem may be hard to find.

Intuitively, to see part 1 of \cref{lem:corblack} we note: If an average reward of $1$ is obtained for a run, the reward machine believes, based on the partial observation of the product MDP, that the run ends up in an ASEC. Almost surely, we eventually discover all possible transitions involving the same state-action pairs as this ASEC and therefore this must also be an ASEC w.r.t.\ the true, unknown product MDP.
For part 2, we modify the policy $\policy$ similarly as in \cref{lem:corgrey} by selecting actions $\act(s_t,q_t)$ once having entered an ASEC $\mathcal C=(T,\act)$ w.r.t.\ the true, unknown product MDP.\footnote{NB The modified policy depends on the true, unknown support of the Probability transition function; we only claim the \emph{existence} of such a policy.}

	\section{Convergence for Limit Average, $\omega$-Regular and LTL Objectives}
	\label{sec:lim-avg}
	
Thanks to the described translation, advances (in both theory and practice) in the study of RL with average rewards carry over to RL with $\omega$-regular and LTL objectives. In this section, we show that it is possible to learn optimal policies for limit average rewards in the limit. Hence, we resolve an open problem \cite{Alur:2022}: also RL with $\omega$-regular and LTL objectives can also be learned in the limit.

We start with the case of simple reward functions $\rma:S\times A\times S\to\Real$.
Recently, \cite{GP23} have shown that discount optimal policies for sufficiently high discount factor $\overline\gamma\in[0,1)$ are also limit average optimal. This is not enough to demonstrate \cref{thm:avgconv} since $\overline\gamma$ is generally not known and in finite time we might only obtain \emph{approximately} limit average optimal policies.

Our approach is to reduce RL with average rewards to a \emph{sequence} of discount sum \changed[dw]{problems} with increasingly high discount factor, which are solved with increasingly high accuracy. 
Our crucial insight is that eventually the approximately optimal solutions to the discounted problems will also be limit average optimal (see \cref{app:lim-avg} for a proof):
\begin{lemma}
    \label{lem:finconv}
    Suppose $\gamma_k\nearrow 1$, $\epsilon_k\searrow 0$ and suppose each $\policy_k$ is a memoryless policy.
    Then there exists $k_0$ such that for all $K\ni k\geq k_0$, $\policy_k$ is limit average optimal, where $K$ is the set of $k\in\nat$ satisfying $\rdisc{\gamma_k}(\policy_k)\geq \rdisc{\gamma_k}(\policy)-\epsilon_k$ for all memoryless policies.
\end{lemma}

Our proof harnesses yet another notion of optimality: a policy $\pi$ is \emph{Blackwell optimal} (cf.~\cite{B62} and \cite[Sec.~8.1]{HY02}) if there exists $\overline\gamma\in (0,1)$ such that $\policy$ is $\gamma$-discount optimal for all $\overline\gamma\leq\gamma<1$.
It is well-known that memoryless Blackwell optimal strategies always exist \cite{B62,GP23} and they are also limit-average optimal \cite{HY02,GP23}.

Thanks to the PAC (probably approximately correct) learnability of RL with discounted rewards \cite{Kearns:2002,SLL09}, there exists an algorithm \texttt{Discounted} which receives as inputs a simulator for $\mdp$, $\rma$ as well as $\gamma,\epsilon$ and $\delta$, and  with probability $1-\delta$ returns an $\epsilon$-optimal memoryless policy for discount factor $\gamma$.
In view of \cref{lem:finconv}, our approach is to run the PAC algorithm for discount-sum RL for increasingly large discount factors $\gamma$ and increasingly low $\delta$ and $\epsilon$ (\cref{alg:avg}).

\begin{algorithm}
	\caption{RL for limit average rewards}\label{alg:avg}
	\begin{algorithmic}
		\Require simulator for $\mdp,  \rma$
		\For{$k\in\nat$}
		\State $\pi_k \gets \texttt{Discounted}(\mdp, \rma, \underbrace{1-1/k}_{\gamma_k}, \underbrace{1/k}_{\epsilon_k}, \underbrace{1/k^2}_{\delta_k})$
        \vspace*{-4mm}
		\EndFor
	\end{algorithmic}
\end{algorithm}

\begin{theorem}
    \label{thm:avgconv}
    RL with average reward functions can be learned in the limit by \cref{alg:avg}: almost surely there exists $k_0\in\nat$ such that $\policy_k$ is limit-average optimal for $k\geq k_0$.
\end{theorem}
\begin{proof}
    Using the definition for $K$ of \cref{lem:finconv} of iterations where the PAC-MDP algorithm succeeds,
    \begin{align*}
        \ex\left[\# (\nat\setminus K) \right]&\leq \sum_{k\in\nat}\prob[\text{PAC-MDP fails in iteration $k$}]\leq \sum_{k\in\nat} \delta_k=\sum_{k\in\nat} \frac 1{k^2} < \infty
    \end{align*}
    The claim follows immediately with  \cref{lem:finconv}.
    \end{proof}
Next, we turn to the more general case of reward \emph{machines}.
\cite{Icarte:PhD,Icarte:18a} observe that optimal policies for reward machines can be learned by learning optimal policies for the modified MDP which additionally tracks the state the reward machine is in and assigns rewards accordingly. We conclude at once: 
\begin{corollary}
    \label{cor:lim}
    RL with average reward machines can be learned in the limit.
\end{corollary}

Finally, harnessing \cref{thm:main} we resolve Open Problem 4 of \cite{Alur:2022}:
\begin{theorem}
    \label{thm:ltlconv}
    RL with $\omega$-regular and LTL objectives can be learned in the limit.
\end{theorem}

\paragraph{Discussion.} \cref{alg:avg} makes independent calls to black box algorithms for discount sum rewards. Many such algorithms with PAC guarantees are model based (e.g.\ \cite{Kearns:2002,SLL09}) and sample from the MDP to obtain suitable approximations of the transition probabilities. Thus, \cref{alg:avg} can be improved in practice by re-using approximations obtained in earlier iterations and refining them.

	\section{Related Work and Conclusion}\label{sec:related-work}

The connection between acceptance of $\omega$-regular languages in the product MDP and AECs is well-known in the field of probabilistic model checking~\cite{Baier:2008,Alfaro:99}.  
As an alternative to DRAs \cite{Wolff:12,Ding:14,Sadigh:2014}, Limit Deterministic B\"{u}chi Automata \cite{Sickert:2016} have been employed to express $\omega$-regular languages for RL \cite{Voloshin:22, Bozkurt:2019, Cai:2023, Hosein:2023, Hasanbeig:2020}.

A pioneering work on RL for $\omega$-regular rewards is~\cite{Wolff:12}, which expresses $\omega$-regular objectives using Deterministic Rabin Automata.   
Similar RL approaches for $\omega$-regular objectives can also be found in~\cite{Ding:14, Voloshin:22, Cai:2023, FuT:14}. The authors of~\cite{FuT:14,Perez2024} approach RL for  $\omega$-regular objectives directly by studying the reachability of AECs in the product MDP and developing variants of the R-MAX algorithm~\cite{Brafman:2003} to find optimal policies. 
However, these approaches require prior knowledge of the MDP, such as the structure of the MDP, the optimal $\epsilon$-return mixing time~\cite{FuT:14}, or the $\epsilon$-recurrence time~\cite{Perez2024}.

Various studies have explored reductions of $\omega$-regular objectives to discounted rewards, and subsequently applied Q-learning and its variants for learning optimal policies~\cite{Bozkurt:2019, Sadigh:2014, Hosein:2023, Hasanbeig:2020, Gao:2019}. 
\changed[dw]{In a similar spirit, \cite{V0Y23} present a translation from LTL objectives to \emph{eventual discounted} rewards, where only strictly positive rewards are discounted. }
These translations are generally not optimality preserving unless the discount factor is selected in a suitable way. 
Again, this is impossible without prior knowledge of the exact probability transition functions in the MDP.


Furthermore, whilst there are numerous convergent RL algorithms for average rewards for \emph{unichain} \changed[dw]{or \emph{communitcating}\footnote{These assumptions generally fail for our setting, where in view of \cref{cor:lim}, MDP states also track the states of the reward machine. For instance, in the reward machine in \cref{fig:rm} it is impossible to reach $u_1$ from $u_2$.}} MDPs (e.g.~\cite{Brafman:2003, Yang:2016,Gosavi:2004,Schwartz:1993,Auer:2008,WNS21}), it is unknown whether such an algorithm exists for general multichain MDPs with a guaranteed convergence property. In fact, a negative result in~\cite{Alur:2022} shows that there is no PAC (probably approximately correct) algorithm for LTL objectives and limit-average rewards when the MDP transition probabilities are unknown. 

\cite{Brafman:2003} have proposed an algorithm with PAC guarantees provided $\epsilon$-return mixing times are known. 
They informally argue that for fixed sub-optimality tolerance $\epsilon$, this assumption can be lifted by guessing increasingly large candidates for the $\epsilon$-return mixing time. This yields $\epsilon$-approximately optimal policies in the limit. However, it is not clear how to asymptotically obtain exactly optimal policies as this would require simultaneously decreasing $\epsilon$ and increasing guesses for the $\epsilon$-return mixing time (which depends on $\epsilon$). 


\paragraph{Conclusion.} We have presented an optimality-preserving translation from $\omega$-regular objectives to limit-average rewards furnished by reward machines. 
As a consequence, off-the-shelf RL algorithms for average rewards
can be employed in conjunction with our translation to learn policies for $\omega$-regular objectives. 
Furthermore, we have developed an algorithm asymptotically learning provably optimal policies for limit-average rewards. Hence, also optimal policies for $\omega$-regular and LTL objectives can be learned in the limit.
Our results provide affirmative answers to two open problems in~\cite{Alur:2022}.


\paragraph{Limitations.}
We focus on MDPs with finite state and action sets and assume states are fully observable.
The assumption of \cref{sec:grey} that the support of the MDP's probability transition function is known is eliminated in \cref{sec:general}. 
Whilst the size of our general translation---the first optimality-preserving translation---is exponential, the additional knowledge in \cref{sec:grey} enables a construction of the reward machine of the same size as the DRA expressing the objective. Hence, we conjecture that this size is minimal.
Since RL with average rewards is not PAC learnable, we cannot possibly provide finite-time complexity guarantees of our \cref{alg:avg}.

\ack{
 This research is supported by the National Research Foundation, Singapore, under its RSS Scheme (NRFRSS2022-009).
}

\bibliographystyle{plain}
\bibliography{main.bib}

\appendix



\newpage

\section{Supplementary Materials for \cref{sub:negative-results}}
\label{app:neg}


 Recall that a $\omega$-regular language $L$ is prefix-independent if for every infinite label sequence $w~\in~(2^\mathcal{AP})^\omega$, we have $w \in L$ iff $w' \in L$ for every suffix $w'$ of $w$.  We prove that there is no optimality-preserving translation for reward functions regardless of whether $L$ is prefix-independent or not. The prefix-dependent case was  given in~\cref{sub:negative-results}. Here we focus on the other case:
  \begin{proposition}\label{prop:fun-general}
  There exists a tuple $(S,A,s_0,\lfun)$  and a prefix-independent $\omega$-regular language $L$ for which it is impossible to find a reward function $\mathcal{R}:S \times A \times S \rightarrow\mathbb{R}$ such that for every  probability transition $P$, let $\mdp = (S,A,s_0,P,\lfun)$, then  every $\mathcal{R}^{\textrm{avg}}$-optimal policy of $\mathcal{M}$ is also $L$-optimal (i.e. maximizing the probability of membership in $L$).
  \end{proposition}
    \begin{proof}
  	Our proof technique is based on the fact that we can modify the  transition probability function. Consider the MDP in~\cref{subfig:counter-mdp2}, where the  objective is to visit either $s_1$ or $s_3$ infinitely often. It can be checked that the DRA in~\cref{subfig:counter-dra2} captures the given objective and the language accepted by $\dra$ is prefix-independent.  There are only two deterministic memoryless policies: $\pi_1$, which consistently selects action $a$, and $\pi_2$, which consistently selects action $b$. For the sake of contradiction, let's assume the existence of a reward function $\mathcal{R}$ that preserves optimality for every transition probability function $P$. Pick $p_1 = 1$ and $p_2 = 0$. Then $\vfun{\mdp}{\dra}(\pi_1) =1$ and $\vfun{\mdp}{\dra}(\pi_2) =0$, which implies that $\pi_1$ is $\dra$-optimal whereas $\pi_2$ is not. Thus $\mathcal{R}(s_1,a,s_1) = \vfun{\mdp}{\mathcal{R}^{\textrm{avg}}}(\pi_1)  >\vfun{\mdp}{\mathcal{R}^{\textrm{avg}}}(\pi_2) = \mathcal{R}(s_0,b,s_0)$.  Now, assume $p_1,p_2 \in (0,1)$. Accordingly, we have $\vfun{\mdp}{\mathcal{R}^{\textrm{avg}}}(\pi_1) \geq p_1 \mathcal{R}(s_1,a,s_1)$ and we can deduce that (e.g. by solving the linear equation system described in~\cite[\S 8.2.3]{Puterman:1994}) $\vfun{\mdp}{\mathcal{R}^{\textrm{avg}}}(\pi_2) = \frac{p_2}{2-p_2}\mathcal{R}(s_0,b,s_0) + \frac{1-p_2}{2-p_2}\left(\mathcal{R}(s_0,b,s_3) + \mathcal{R}(s_3,b,s_0)\right)$. As a result:
  	\[
  	\lim_{p_1 \rightarrow 1} \vfun{\mdp}{\mathcal{R}^{\textrm{avg}}}(\pi_1) ~\geq~  \mathcal{R}(s_1,a,s_1) ~>~\mathcal{R}(s_0,b,s_0)  ~=~ \lim_{p_2 \rightarrow 1} \vfun{\mdp}{\mathcal{R}^{\textrm{avg}}}(\pi_2)
  	\]
  	
  	Consequently, if $p_1,p_2$ are sufficiently large then $\vfun{\mdp}{\mathcal{R}^{\textrm{avg}}}(\pi_1) >  \vfun{\mdp}{\mathcal{R}^{\textrm{avg}}}(\pi_2)$. However, this contradicts to the fact that $\pi_2$ is $\dra$-optimal and $\pi_1$ is not, since $\vfun{\mdp}{\dra}(\pi_2) =1 > p_1 = \vfun{\mdp}{\dra}(\pi_1)$. Hence, there is no such reward function $\mathcal{R}$.
  \end{proof}

  
  	\begin{figure}
  	\centering
  	\begin{subfigure}{0.55\linewidth}
  	\centering
  	\begin{tikzpicture}
  		[shorten >=1pt,node distance=2.5cm,on grid,auto]
  		\node[state, initial] (s0) {$s_0$};
  		\node[state,  below of=s0] (s1) {$s_1$};
  		\node[state,  right of=s0] (s3) {$s_3$};
  		\node[state, below of=s3] (s2) {$s_2$};
  		\path[->] 
  		(s1) edge[loop right] node{$a/1$} (s1)
  		(s2) edge[loop right] node{$a/1$} (s2)
  		(s0) edge[loop above] node{$b/p_2$} (s0)
  		(s0) edge[right] node{$a/p_1$} (s1)
  		(s0) edge[right] node{$a/1-p_1$} (s2)
  		(s0) edge[bend left, above] node{$b/1-p_2$} (s3)
  		(s3) edge[bend left, above] node{$b/1$} (s0);
  	\end{tikzpicture}	
  	\caption{An MDP $\mdp$ where $\lfun(s_1,a,s_1) = \lfun(s_3,b,s_0) = \{c\}$, and the rest are labeled with $\emptyset$.}
  	\label{subfig:counter-mdp2}
  	\end{subfigure}
  	\hfill
  	\begin{subfigure}{0.42\linewidth}
  			\centering
  		\begin{tikzpicture}
  			[shorten >=1pt,node distance=2.5cm,on grid,auto]
  			\node[state, initial] (s0) {$q_0$};
  			\node[state,  right of=s0] (s1) {$q_1$};
  			\path[->] 
  			(s0) edge[bend left,above] node{$\{c\}$} (s1)
  			(s0) edge[loop above] node{$\emptyset$} (s0)
  			(s1) edge[loop above] node{$\{c\}$} (s1)
  			(s1) edge[bend left,below] node{$\emptyset$} (s0);
  		\end{tikzpicture}	
  		\caption{A DRA $\dra$ for the objective of visiting $s_1$ or $s_3$ infinitely often where $F\defeq\{(\{q_1\},\emptyset)\}$.}
  		\label{subfig:counter-dra2}
  	\end{subfigure}
  	\caption{Counter-example for prefix-independent objectives.}\label{fig:counter-2}
  \end{figure}
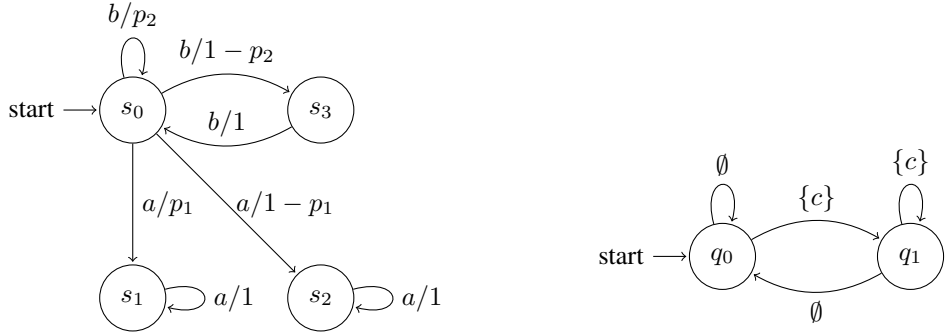

\section{Supplementary Materials for \cref{sec:grey}}
\label{app:grey}

  \begin{lemma}\label{lem:asec}
    Every AEC contains an ASEC. 
\end{lemma}
\begin{proof}Consider an AEC $\mathcal{C} = (T,\act)$ of $\mdp_{\mathcal{A}}$. We will prove this by using induction on the number of actions in $\mathcal{C}$, denoted as $\mathsf{size}(\mathcal{C})\defi \sum_{s \in T} |\act(s)| \geq 1$. For the base case where $\mathsf{size}(\mathcal{C}) = 1$, it can be deduced that $\mathcal{C}$ consists of only one accepting state with a self-loop. Therefore, $\mathcal{C}$ itself is an ASEC.
    
    Now, let's assume that $\mathsf{size}(\mathcal{C}) = k + 1 \geq 2$. If $\mathcal{C}$ is already an ASEC, then we are done. Otherwise, there exists a state $s \in T$ such that $|\act(s)| > 1$. Since $\mathcal{C}$ is strongly connected, there exists a finite path $\run = s a s_1 a_1 \ldots s_n  a_n s_F$ where $s_F$ is an accepting state and all the states $s_1,\ldots,s_n$  are different from $s$. Let $a' \in \act(s)$ such that $a' \neq a$. We construct a new AEC $\mathcal{C}' = (T',\act')$ by first removing $a'$ from $\act(s)$ and then removing all the states that are no longer reachable from $s$ along with their associated transitions. It is important to note that after the removal, $s_F \in T'$ since we can reach $s_F$ from $s$ without taking the action $a'$. (Besides, the graph is still strongly connected.) Since $\mathsf{size}(\mathcal{C}') \leq k$, we can apply the induction hypothesis to conclude that $\mathcal{C}'$ contains an ASEC, thus completing the proof.
\end{proof}

\asecsuff*

To proof this result, we recall  a well-known result in probabilistic model checking that with probability of one (wpo), every run $\run$ of the policy $\pi$  eventually stays  in one of the ECs of $\mdp_{\mathcal{A}}$ and visits every transition in that EC infinitely often. 
To state this formally, we define for any run $\run=s_0a_0s_1\cdots$,
\begin{align*}
    \inft(\run)\defi\{(s,a)\in S\times A\mid |\{i\in\nat\mid s_i=s\land a_i=a\}|=\infty\}
\end{align*}
the set of state-action-pairs occurring infinitely often in $\run$. Furthermore, a state-action set $\chi\subseteq S\times A$ defines a sub-MDP $\sub(\chi)\defeq(T,\act)$, where
\begin{align*}
    T&\defeq\{s\in S\mid (s,a)\in\chi\text{ for some }a\in A\}&
    \act(s)&\defeq\{a\mid (s,a)\in\chi\}
\end{align*}

\begin{restatable}[\cite{Alfaro:99}]{lemma}{ec}
    \label{lem:ec}$ 
     \prob_{\run \sim \dist^{\prodmdp}_\pi}[\sub(\inft(\run))\text{ is an end component}] = 1$.
 \end{restatable}
 For the sake of self-containedness, we recall the proof of \cite{Alfaro:99}.
\begin{proof}
    We start with two more definitions: for any sub-MDP $(T,\act)$ \cite{Alfaro:PhD}, let
    \begin{align*}
        \sa(T,\act)\defi\{(s,a)\in T\times A\mid a\in \act(s)\}
    \end{align*}
    be the set of state-action pairs $(s,a)$ such that $a$ is enabled in $s$. Finally, let
    \begin{align*}
        \runssmdp^{(T,\act)}\defi\{\run\in\runs\mid \inft(\run)=\sa(T,\act)\}
    \end{align*}
    be the set of runs such that action $a$ is taken infinitely often in state $s$ iff $s\in T$ and $a\in\act(s)$.
    Note that the $\runssmdp^{(T,\act)}$ constitute a partition of $\runs$.
    
    Therefore, it suffices to establish for any sub-MDP $(T,\act)$, $(T,\act)$ is an end-component or $\prob[\run\in\runssmdp^{(T,\act)}]=0$.

    Let $(T,\act)$ be an arbitrary sub-MDP.
    First, suppose there exist $s\in T$ and $a\in\act(t)$ such that $p\defi\sum_{s'\in T} \trans(t,a,t')<1$.
    By definition each $\run\in\runssmdp^{(T,\act)}$ takes action $a$ in state $s$ infinitely often. Hence, not only $\prob[\run\in\runssmdp^{(T,\act)}]\leq p^k$ for all $k\in\nat$ but also $\prob[\run\in\runssmdp^{(T,\act)}]=0$.

    Thus, we can assume that for all $s\in T$ and $a\in\act(t)$, $\sum_{s'\in T} \trans(t,a,t')=1$. If $\runssmdp^{(T,\act)}=\emptyset$ then clearly $\prob[\run\in\runssmdp^{(T,\act)}]=0$ follows. Otherwise, take any $\run=s_0a_0a_1\cdots\in\runssmdp^{(T,\act)}$, and let $t,t'\in T$ be arbitrary. We show that there exists a connecting path in $(T,\to_{\act})$, which implies that $(T,\act)$ is an end component.

    Evidently, there exists an index $i_0$ such that all state-action pairs occur infinitely often in $\run$, i.e.
    \begin{align*}
        \{(s_{i_0},a_{i_0}),(s_{i_0+1},a_{i_0+1}),\ldots\}=\inft(\run)
    \end{align*}
    Thus, for all $i\geq i_0$, $s_i\in T$ and $a_i\in\act(s_i)$, and for all $i'>i\geq i_0$, there is a path from $s_i$ to $s_{i'}$ in $(T,\to_{\act})$. Finally, it suffices to note that clearly for some $i'>i=i_0$, $s_i=t$ and $s_{i'}=t'$. 
\end{proof}


\begin{proof}[Proof of \cref{lem:asecsuff}]
    By \cref{lem:ec}, almost surely $\sub(\inft(\run))$ is an accepting end component. Clearly, $\run$ is only accepted by the product MDP if this end component is an \emph{accepting} EC. By \cref{lem:asec} this AEC contains an ASEC. Therefore, by definition of $\sub(\inft(\run))$, $\run$ almost surely in particular \emph{enters} some ASEC. Finally, since the $\mathcal C_1,\ldots,\mathcal C_n$ cover all states in ASECs, $\run$ almost surely enters some $\mathcal C_i$.
\end{proof}

\begin{figure}
	\centering
	\begin{tikzpicture}
		[shorten >=1pt,node distance=3cm,on grid,auto]
		\node[state, initial] (s0) {$q_0$};
		\node[state, right=30mm of s0] (s1) {$q_1$};
		\node[state,  right=30mm of s1] (bot) {$\bot$};
		\node[state,  right=30mm of bot] (s2) {$q_2$};
		\path[->] 
		(s0) edge[loop above] node{$(s_0,a,s_0)/0$} (s0)
		(s0) edge[loop below] node{$(s_1,b,s_0)/0$} (s0)
		(s0) edge[above] node{$(s_0,b,s_1)/0$} (s1)
		(s1) edge[loop above] node{$(s_0,a,s_0)/1$} (s1)
		(s1) edge[loop below] node{$(s_1,b,s_0)/0$} (s1)
		(s1) edge[above] node{$(s_0,b,s_1)/{\color{red}1}$} (bot)
		(bot) edge[loop above] node{$*/0$} (bot)
		(s2) edge[loop above] node{$*/0$} (s2);
	\end{tikzpicture}
	\caption{Reward machine yielded by our construction in \cref{sec:grey} for the running example.}
	\label{fig:construction}
\end{figure}
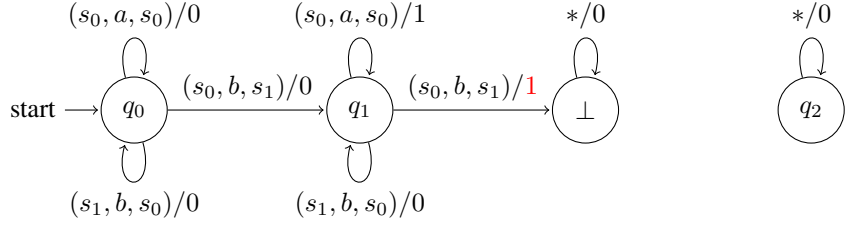

Before turning to the proof of \cref{lem:corgrey}, 
 let $\rlavg(\run)=\liminf_{t\to\infty}\frac{1}{t}\cdot \sum_{i=0}^{t-1}~r_i$ denote the limit-average reward of a run $\run$. Note that, for any run $\run$, $\rlavg(\run)\in\{0,1\}$. Thus, by the dominated convergence theorem \cite[Cor.~6.26]{K13}, 
\begin{align}
	\label{eq:exliminf}
	\prob_{\run\sim\dist_\pi^\mdp}\left[\rlavg(\run)=1\right]~=~\ex_{\run\sim\dist_\pi^\mdp}[\rlavg(\run)]~=~\liminf_{t\to\infty}\ex_{\run\sim\dist_\pi^\mdp}\left[\frac{1}{t}\cdot \sum_{i=0}^{t-1}~r_i\right]~=~\rlavg(\pi)
\end{align}

\corgrey*
\begin{proof}
	\begin{enumerate}
		\item For any run $\run$, $\rlavg(\run)=1$ only if $\run^\otimes$ enters a $\mathcal C_i$ and never leaves it. ($\run^\otimes$ might have entered other $\mathcal C_j$'s earlier but then those necessarily need to overlap with yet another $\mathcal C_k$ such that $i\leq k<j$ to avoid being trapped in state $\bot$, resulting in $\rlavg(\run)=1$. Furthermore, this $\mathcal C_i$ can only overlap with $\mathcal C_j$ if $i<j$. Otherwise, the reward machine would have enforced transitioning to $\mathcal C_j$.) \dw{more elaboration necessary?}

		Since $\mathcal C_i$ is an ASEC, $\run^\otimes$ is accepted by the product MDP $\prodmdp$. Hence, by \cref{eq:prodmdp,eq:exliminf},
		\begin{align*}
			\rlavg(\pi)&~=~\prob_{\run\sim\dist_\pi^\mdp}\left[\rlavg(\run)=1\right]
			~\leq~\prob_{\run\sim\dist_\pi^\mdp}\left[\run^\otimes\text{ accepted by }\prodmdp\right]
			~=~\racc(\pi) 
		\end{align*}

		\item Let $\policy$ be arbitrary. For a run $s_0a_0\cdots$ let $q_t$ be the state of the DRA in step $t$. Define $\policy'$ to follow $\policy$ until reaching $s_t$ such that $(s_t,q_t)\in T_1\cup\cdots \cup T_n$. Henceforth, we select the (unique) action guaranteeing to stay in the $\mathcal C_i$ with minimal $i$ including the current state, i.e. $\act_{(q,u)}(q,u)$. Formally\footnote{We slightly abuse notation in the ``otherwise''-case and denote by $\act_{(s_t,q_t)}(s_t,q_t)$ the distribution selecting the state in the singleton set $\act_{(s_t,q_t)}(s_t,q_t)$ with probability 1.},
		\begin{align}
			\policy'(s_0a_0\cdots s_t)~\defi~
			\begin{cases}
				\policy(s_0a_0\cdots s_t)&\text{if }(s_t,q_t)\not\in T_1 \cup \cdots \cup T_n\\
				\act_{(s_t,q_t)}(s_t,q_t)&\text{otherwise}
			\end{cases}
			\label{eq:surgery}
		\end{align} 
		Note that whenever a run $\run\sim\dist_{\pi'}^\mdp$ follows the modified policy $\pi'$ and its induced run $\run^\otimes$ reaches some ASEC $\mathcal C_i$ then $\rlavg(\run)=1$. Thus, 
		\begin{align*}
			\prob_{\run\sim\dist_{\policy'}^\mdp}[\run^\otimes \text{ reaches some }\mathcal C_i]
			~\leq~\ex_{\run\sim\dist_{\policy'}^\mdp}[\rlavg(\run)]~=~\rlavg(\pi')
		\end{align*}
		Furthermore, by \cref{lem:asecsuff} almost surely, every induced run $\run^\otimes$ accepted by the product MDP must reach some $\mathcal C_i$. Consequently, by \cref{eq:prodmdp},
		\begin{align*}
			\racc(\pi)
			&=\prob_{\run\sim\dist_\policy^\mdp}[\run^\otimes\text{ is accepted by }\prodmdp]\\
			&\leq\prob_{\run\sim\dist_\policy^\mdp}[\run^\otimes\text{ reaches some }\mathcal C_i]\\
			&=\prob_{\run\sim\dist_{\policy'}^\mdp}[\run^\otimes\text{ reaches some }\mathcal C_i]
			\leq\rlavg(\pi')
		\end{align*}
		In the penultimate step, we have exploited the fact that $\pi$ and $\pi'$ agree until reaching the first $\mathcal C_i$.\qedhere
	\end{enumerate}
\end{proof}

\fi

\subsection{Efficient Construction}
\label{sec:eff}

\changed[dw]{
We consider a different collection $\mathcal C_1,\ldots,\mathcal C_n$ of ASECs:
\begin{quote}
	Suppose $\mathcal C'_1,\ldots,\mathcal C'_n$ is a collection of AECs (not necessarily simple ones) containing all states in AECs. Then we consider ASECs $\mathcal C_1,\ldots,\mathcal C_n$ such that $\mathcal C_i$ is contained in $\mathcal C'_i$.
\end{quote}

The definition of the reward machine in \cref{sec:rmc} and the extension in \cref{sec:general} do not need to be changed. Next, we argue the following:
\begin{enumerate}
	\item This collection can be obtained efficiently (in time polynomial in the size of the MDP and DRA).
	\item \cref{lem:corgrey} and hence the correctness result (\cref{thm:grey}) still hold.
\end{enumerate}
For 1. it is well-known that a collection 
 of maximal AECs (covering all states in AECs) can be found efficiently using graph algorithms \cite[Alg. 3.1]{Alfaro:PhD}, \cite{FuT:14,CH11} and \cite[Alg. 47 and Lemma 10.125]{Baier:2008}. Subsequently, \cref{lem:asec} can be used to obtain an ASEC contained in each of them. In particular, note that the proof of \cref{lem:asec} immediately gives rise to an efficient algorithm. (Briefly, we iteratively remove actions and states whilst querying reachability properties.)

For 2., the first part of \cref{lem:corgrey} clearly still holds. For the second, we modify policy $\policy$ as follows: Once, $\policy$ enters a maximal accepting end component we select an action on the shortest path to the respective ASEC $\mathcal C_i$ inside $\mathcal C'_i$. Once we enter one of the $\mathcal C_i$ we follow the actions specified by the ASEC as before. Observe that the probability that under 
 an AEC is entered is the same as the probability that one of the $\mathcal C_i$ is entered under the modified policy. The lemma, hence \cref{thm:grey}, follow.}

\section{Supplementary Materials for \cref{sec:general}}
\label{app:general}

\corblack*
\begin{proof}
	\begin{enumerate}
		\item For a run $\run$, let $E_\run$ be the set of transitions encountered in the product MDP. Note that $\rlavg(\run)=1$ only if $\run^\otimes$ enters some $\mathcal C^{E_\run}_i$ and never leaves it. ($\run^\otimes$ might have entered other $\mathcal C^E_j$s earlier for $E\subseteq E_\run$.) \dw{more elaboration necessary?}

		With probability 1, $E_\run$ contains all the transitions present in $\mathcal C^{E_\run}_i$ in the actual MDP. \dw{clear what is meant?}\lxb{Only true if you consider `normal' $\zeta$ which eventually ends up in some EC} \dw{Yes, but that happens with probability 1. Do you agree?} (NB possible transitions outside of $\mathcal C^{E_\run}_i$ might be missing from $E_\run$.)
		In particular, with probability 1, $\mathcal C^{E_\run}_i$ is also an ASEC for the true unknown MDP and $\run^\otimes$ is accepted by the product MDP $\prodmdp$. Consequently, using \cref{eq:exliminf} again,
			\begin{align*}
				\rlavg(\pi)&=\prob_{\run\sim\dist_\pi^\mdp}[\rlavg(\run)=1]
				\leq\prob_{\run\sim\dist_\pi^\mdp}[\run^\otimes\text{ accepted by }\prodmdp]
				=\racc(\pi)
			\end{align*}

		\item Let $\policy$ be arbitrary. We modify $\policy$ to $\policy'$ as follows:
		until reaching an ASEC $\mathcal C=(T,\act)$ w.r.t.\ the true, unknown\footnote{NB The modified policy depends on the true, unknown $E^*$; we only claim the \emph{existence} of such a policy.} set of transitions $E^*$ follow $\pi$. Henceforth, select action $\act^{E^*}_{(s_t,q_t)}(s_t,q_t)$.

		We claim that whenever $\run\sim\dist_{\pi'}^\mdp$ follows the modified policy $\policy'$ and $\run^\otimes$ reaches some ASEC in the true product MDP, $\rlavg(\run)=1$. \dw{elaborate}
	
		To see this, suppose $\run\sim\dist_{\pi'}^\mdp$ is such that for some minimal $t_0\in\nat$, $(s_{t_0},q_{t_0})\in T_1^{E^*}\cup\cdots\cup T_n^{E^*}$. Let $\mathcal C=(T,\act)\defeq\mathcal C^{E^*}_{(s_{t_0},q_{t_0})}$.
		
		Define $E_t$ to be the transitions encountered up to step $t\in\nat$, i.e.\ $E_t\defeq\{((s_k,q_k),a_k,(s_{k+1},q_{k+1}))\mid 0\leq k<t\}$.
		Then almost surely for some minimal $t\geq{t_0}$, $E_{t}$ contains all transitions in $\mathcal C$, and no further transitions will be encountered, i.e.\ for all $t'\geq t$, $E_{t'}=E_{t}$. Define $\overline E\defeq E_{t}$. \dw{ASECs don't contain more ASECs} Note that for all $((s,q),a,(s',q'))\in\overline E$ such that $(s,q)\in T$, $\act(s,q)=\{a\}$. (This is because upon entering the ASEC $\mathcal C$ we immediately switch to following the action dictated by $\act$. Thus, we avoid ``accidentally'' discovering other ASECs w.r.t.\ the partial knowledge of the product MDP's graph, which might otherwise force us to perform actions leaving $\mathcal C$.) Consequently, there cannot be another ASEC $\mathcal C'=(T',\act')$ w.r.t.\ $\overline E$ overlapping with $\mathcal C$, i.e.\ $T\cap T'\neq\emptyset$\lxb{I think you can make this statement stronger by saying that $\overline E$ does not contain any other ASEC beside $\mathcal{C}$}. Therefore, for all $(s,q)\in\mathcal C$, $\act_{(s,q)}^{\overline E}=\act$. Consequently, $\rlavg(\run)=1$. \dw{Is this convincing?}

		 Thus, 
		\begin{align*}
			\prob_{\run\sim\dist_{\policy'}^\mdp}[\run^\otimes \text{ reaches some ASEC in true product MDP}]
			\leq\ex_{\run\sim\dist_{\policy'}^\mdp}[\rlavg(\run)]=\rlavg(\pi')
		\end{align*}
		 Consequently,
		\begin{align*}
			\racc(\pi)
			&=\prob_{\run\sim\dist_\policy^\mdp}[\run^\otimes\text{ is accepted by }\prodmdp]\\
			&\leq\prob_{\run\sim\dist_\policy^\mdp}[\run^\otimes\text{ reaches some ASEC in true product MDP}]\\
			&=\prob_{\run\sim\dist_{\policy'}^\mdp}[\run^\otimes\text{ reaches some ASEC in true product MDP}]
			\leq\rlavg(\pi')
		\end{align*}
		In the penultimate step we have exploited that $\pi$ and $\pi'$ agree until reaching some ASEC in true product MDP.\qedhere
	\end{enumerate}
\end{proof}

\section{Supplementary Materials for \cref{sec:lim-avg}}
\label{app:lim-avg}
\dw{Notation:} Let $\Pi$ be the set of all memoryless policies and $\Pi^*$ be the set of all limit-average optimal policies. Besides, let $w^*\defeq \rlavg(\policy^*)$ the limit average reward of any optimal $\policy^*\in\Pi^*$.

\cref{lem:finconv} is proven completely analagously to the following (where $K=\nat$):
\begin{lemma}
    \label{lem:approx}
    Suppose $\gamma_k\nearrow 1$, $\epsilon_k\searrow 0$ and each $\policy_k$ is a memoryless policy satisfying $\rdisc{\gamma_k}(\policy_k)\geq \rdisc{\gamma_k}(\policy)-\epsilon_k$ for all $\policy\in\Pi$. Then there exists $k_0$ such that for all $k\geq k_0$, $\policy_k$ is limit average optimal.
\end{lemma}
\begin{proof}
    We define $\Delta\defeq\min_{\policy\in\Pi\setminus\Pi^*} \rlavg(\policy)-w^*>0$. 
    Recall (see e.g.\ \cite[Sec.~8.1]{HY02}) that for any policy $\policy\in\Pi$,
\begin{align}
    \label{eq:discavg}
    \lim_{\gamma\nearrow 1}(1-\gamma)\cdot \rdisc\gamma(\policy)=\rlavg(\policy)
\end{align}
Since $\Pi$ is finite, due to \cref{eq:discavg} there exists $\gamma_0$ such that
    \begin{align}
        \label{eq:discapp}
        |\rlavg(\policy)-(1-\gamma)\cdot \rdisc\gamma(\policy)|\leq\frac\Delta 4
    \end{align}
    for all $\policy\in\Pi$ and $\gamma\in[\gamma_0,1)$. 
    Let $\policy^*$ be a memoryless Blackwell optimal policy (which exists due to \cite{B62,GP23}). Note that 
    \begin{align}
        \label{eq:optpo}
        w^*&=\rlavg(\policy^*)
    \end{align}
    and there exists $\overline\gamma\in[0,1)$ such that 
    \begin{align}
        \label{eq:optdisc}
        \rdisc\gamma(\policy^*)\geq \rdisc\gamma(\policy)
    \end{align}
    for all $\gamma\in[\overline\gamma,1)$ and $\policy\in\Pi$.
    Moreover, there clearly exists $k_0$ such that $\epsilon_k\leq\Delta/4$ and $\gamma_k\geq\gamma_0,\overline\gamma$ for all $k\geq k_0$. 
    
    Therefore, for any $k\geq k_0$,
    \begin{align*}
        |\rlavg(\policy_k)-w^*|&\leq (1-\gamma_k)\cdot \left|\rdisc{\gamma_k}(\policy_k)-\rdisc{\gamma_k}(\policy^*)\right|+\frac\Delta 2&\text{\cref{eq:discapp,eq:optpo}}\\
        &\leq (1-\gamma_k)\cdot \epsilon_k+\frac\Delta 2&\text{premise and \cref{eq:optdisc}}\\
        &\leq\frac 4 3\cdot\Delta
    \end{align*}
    Consequently, by definition of $\Delta$, $\pi_k\in\Pi^*$.
\end{proof}



\newpage
\section*{NeurIPS Paper Checklist}

\begin{enumerate}
	
	\item {\bf Claims}
	\item[] Question: Do the main claims made in the abstract and introduction accurately reflect the paper's contributions and scope?
	\item[] Answer: \answerYes{} 
	\item[] Justification: The main results mentioned in the abstract and introduction are \cref{thm:r1,thm:main,thm:avgconv,thm:ltlconv}. They accurately reflect the paper's contributions and scope.
	\item[] Guidelines:
	\begin{itemize}
		\item The answer NA means that the abstract and introduction do not include the claims made in the paper.
		\item The abstract and/or introduction should clearly state the claims made, including the contributions made in the paper and important assumptions and limitations. A No or NA answer to this question will not be perceived well by the reviewers. 
		\item The claims made should match theoretical and experimental results, and reflect how much the results can be expected to generalize to other settings. 
		\item It is fine to include aspirational goals as motivation as long as it is clear that these goals are not attained by the paper. 
	\end{itemize}
	
	\item {\bf Limitations}
	\item[] Question: Does the paper discuss the limitations of the work performed by the authors?
	\item[] Answer: \answerYes{} 
	\item[] Justification: Limitations are discussed in \cref{sec:related-work}. \dw{todo: elaborate}
	\item[] Guidelines:
	\begin{itemize}
		\item The answer NA means that the paper has no limitation while the answer No means that the paper has limitations, but those are not discussed in the paper. 
		\item The authors are encouraged to create a separate "Limitations" section in their paper.
		\item The paper should point out any strong assumptions and how robust the results are to violations of these assumptions (e.g., independence assumptions, noiseless settings, model well-specification, asymptotic approximations only holding locally). The authors should reflect on how these assumptions might be violated in practice and what the implications would be.
		\item The authors should reflect on the scope of the claims made, e.g., if the approach was only tested on a few datasets or with a few runs. In general, empirical results often depend on implicit assumptions, which should be articulated.
		\item The authors should reflect on the factors that influence the performance of the approach. For example, a facial recognition algorithm may perform poorly when image resolution is low or images are taken in low lighting. Or a speech-to-text system might not be used reliably to provide closed captions for online lectures because it fails to handle technical jargon.
		\item The authors should discuss the computational efficiency of the proposed algorithms and how they scale with dataset size.
		\item If applicable, the authors should discuss possible limitations of their approach to address problems of privacy and fairness.
		\item While the authors might fear that complete honesty about limitations might be used by reviewers as grounds for rejection, a worse outcome might be that reviewers discover limitations that aren't acknowledged in the paper. The authors should use their best judgment and recognize that individual actions in favor of transparency play an important role in developing norms that preserve the integrity of the community. Reviewers will be specifically instructed to not penalize honesty concerning limitations.
	\end{itemize}
	
	\item {\bf Theory Assumptions and Proofs}
	\item[] Question: For each theoretical result, does the paper provide the full set of assumptions and a complete (and correct) proof?
	\item[] Answer: \answerYes{} 
	\item[] Justification: Full proofs are presented in the appendices and results are cross-referenced. At the beginning of \cref{sec:grey} we assume knowledge of the support of the MDP's probability transition function for presentational purposes. This assumption is fully removed in \cref{sec:general}.
	\item[] Guidelines:
	\begin{itemize}
		\item The answer NA means that the paper does not include theoretical results. 
		\item All the theorems, formulas, and proofs in the paper should be numbered and cross-referenced.
		\item All assumptions should be clearly stated or referenced in the statement of any theorems.
		\item The proofs can either appear in the main paper or the supplemental material, but if they appear in the supplemental material, the authors are encouraged to provide a short proof sketch to provide intuition. 
		\item Inversely, any informal proof provided in the core of the paper should be complemented by formal proofs provided in appendix or supplemental material.
		\item Theorems and Lemmas that the proof relies upon should be properly referenced. 
	\end{itemize}
	
	\item {\bf Experimental Result Reproducibility}
	\item[] Question: Does the paper fully disclose all the information needed to reproduce the main experimental results of the paper to the extent that it affects the main claims and/or conclusions of the paper (regardless of whether the code and data are provided or not)?
	\item[] Answer: \answerNA{} 
	\item[] Justification: The paper does not include experiments.
	\item[] Guidelines:
	\begin{itemize}
		\item The answer NA means that the paper does not include experiments.
		\item If the paper includes experiments, a No answer to this question will not be perceived well by the reviewers: Making the paper reproducible is important, regardless of whether the code and data are provided or not.
		\item If the contribution is a dataset and/or model, the authors should describe the steps taken to make their results reproducible or verifiable. 
		\item Depending on the contribution, reproducibility can be accomplished in various ways. For example, if the contribution is a novel architecture, describing the architecture fully might suffice, or if the contribution is a specific model and empirical evaluation, it may be necessary to either make it possible for others to replicate the model with the same dataset, or provide access to the model. In general. releasing code and data is often one good way to accomplish this, but reproducibility can also be provided via detailed instructions for how to replicate the results, access to a hosted model (e.g., in the case of a large language model), releasing of a model checkpoint, or other means that are appropriate to the research performed.
		\item While NeurIPS does not require releasing code, the conference does require all submissions to provide some reasonable avenue for reproducibility, which may depend on the nature of the contribution. For example
		\begin{enumerate}
			\item If the contribution is primarily a new algorithm, the paper should make it clear how to reproduce that algorithm.
			\item If the contribution is primarily a new model architecture, the paper should describe the architecture clearly and fully.
			\item If the contribution is a new model (e.g., a large language model), then there should either be a way to access this model for reproducing the results or a way to reproduce the model (e.g., with an open-source dataset or instructions for how to construct the dataset).
			\item We recognize that reproducibility may be tricky in some cases, in which case authors are welcome to describe the particular way they provide for reproducibility. In the case of closed-source models, it may be that access to the model is limited in some way (e.g., to registered users), but it should be possible for other researchers to have some path to reproducing or verifying the results.
		\end{enumerate}
	\end{itemize}

	\item {\bf Open access to data and code}
	\item[] Question: Does the paper provide open access to the data and code, with sufficient instructions to faithfully reproduce the main experimental results, as described in supplemental material?
	\item[] Answer: \answerNA{} 
	\item[] Justification: The paper does not include experiments.
	\item[] Guidelines:
	\begin{itemize}
		\item The answer NA means that paper does not include experiments requiring code.
		\item Please see the NeurIPS code and data submission guidelines (\url{https://nips.cc/public/guides/CodeSubmissionPolicy}) for more details.
		\item While we encourage the release of code and data, we understand that this might not be possible, so “No” is an acceptable answer. Papers cannot be rejected simply for not including code, unless this is central to the contribution (e.g., for a new open-source benchmark).
		\item The instructions should contain the exact command and environment needed to run to reproduce the results. See the NeurIPS code and data submission guidelines (\url{https://nips.cc/public/guides/CodeSubmissionPolicy}) for more details.
		\item The authors should provide instructions on data access and preparation, including how to access the raw data, preprocessed data, intermediate data, and generated data, etc.
		\item The authors should provide scripts to reproduce all experimental results for the new proposed method and baselines. If only a subset of experiments are reproducible, they should state which ones are omitted from the script and why.
		\item At submission time, to preserve anonymity, the authors should release anonymized versions (if applicable).
		\item Providing as much information as possible in supplemental material (appended to the paper) is recommended, but including URLs to data and code is permitted.
	\end{itemize}

	\item {\bf Experimental Setting/Details}
	\item[] Question: Does the paper specify all the training and test details (e.g., data splits, hyperparameters, how they were chosen, type of optimizer, etc.) necessary to understand the results?
	\item[] Answer: \answerNA{} 
	\item[] Justification: The paper does not include experiments.
	\item[] Guidelines:
	\begin{itemize}
		\item The answer NA means that the paper does not include experiments.
		\item The experimental setting should be presented in the core of the paper to a level of detail that is necessary to appreciate the results and make sense of them.
		\item The full details can be provided either with the code, in appendix, or as supplemental material.
	\end{itemize}
	
	\item {\bf Experiment Statistical Significance}
	\item[] Question: Does the paper report error bars suitably and correctly defined or other appropriate information about the statistical significance of the experiments?
	\item[] Answer: \answerNA{} 
	\item[] Justification: The paper does not include experiments.
	\item[] Guidelines:
	\begin{itemize}
		\item The answer NA means that the paper does not include experiments.
		\item The authors should answer "Yes" if the results are accompanied by error bars, confidence intervals, or statistical significance tests, at least for the experiments that support the main claims of the paper.
		\item The factors of variability that the error bars are capturing should be clearly stated (for example, train/test split, initialization, random drawing of some parameter, or overall run with given experimental conditions).
		\item The method for calculating the error bars should be explained (closed form formula, call to a library function, bootstrap, etc.)
		\item The assumptions made should be given (e.g., Normally distributed errors).
		\item It should be clear whether the error bar is the standard deviation or the standard error of the mean.
		\item It is OK to report 1-sigma error bars, but one should state it. The authors should preferably report a 2-sigma error bar than state that they have a 96\% CI, if the hypothesis of Normality of errors is not verified.
		\item For asymmetric distributions, the authors should be careful not to show in tables or figures symmetric error bars that would yield results that are out of range (e.g. negative error rates).
		\item If error bars are reported in tables or plots, The authors should explain in the text how they were calculated and reference the corresponding figures or tables in the text.
	\end{itemize}
	
	\item {\bf Experiments Compute Resources}
	\item[] Question: For each experiment, does the paper provide sufficient information on the computer resources (type of compute workers, memory, time of execution) needed to reproduce the experiments?
	\item[] Answer: \answerNA{} 
	\item[] Justification: The paper does not include experiments.
	\item[] Guidelines:
	\begin{itemize}
		\item The answer NA means that the paper does not include experiments.
		\item The paper should indicate the type of compute workers CPU or GPU, internal cluster, or cloud provider, including relevant memory and storage.
		\item The paper should provide the amount of compute required for each of the individual experimental runs as well as estimate the total compute. 
		\item The paper should disclose whether the full research project required more compute than the experiments reported in the paper (e.g., preliminary or failed experiments that didn't make it into the paper). 
	\end{itemize}
	
	\item {\bf Code Of Ethics}
	\item[] Question: Does the research conducted in the paper conform, in every respect, with the NeurIPS Code of Ethics \url{https://neurips.cc/public/EthicsGuidelines}?
	\item[] Answer: \answerYes{} 
	\item[] Justification: The research conducted in the paper conform, in every respect, with the NeurIPS Code of Ethics \url{https://neurips.cc/public/EthicsGuidelines}.
	\item[] Guidelines:
	\begin{itemize}
		\item The answer NA means that the authors have not reviewed the NeurIPS Code of Ethics.
		\item If the authors answer No, they should explain the special circumstances that require a deviation from the Code of Ethics.
		\item The authors should make sure to preserve anonymity (e.g., if there is a special consideration due to laws or regulations in their jurisdiction).
	\end{itemize}

	\item {\bf Broader Impacts}
	\item[] Question: Does the paper discuss both potential positive societal impacts and negative societal impacts of the work performed?
	\item[] Answer: \answerNA{} 
	\item[] Justification: There is no societal impact of the work performed.
	\item[] Guidelines:
	\begin{itemize}
		\item The answer NA means that there is no societal impact of the work performed.
		\item If the authors answer NA or No, they should explain why their work has no societal impact or why the paper does not address societal impact.
		\item Examples of negative societal impacts include potential malicious or unintended uses (e.g., disinformation, generating fake profiles, surveillance), fairness considerations (e.g., deployment of technologies that could make decisions that unfairly impact specific groups), privacy considerations, and security considerations.
		\item The conference expects that many papers will be foundational research and not tied to particular applications, let alone deployments. However, if there is a direct path to any negative applications, the authors should point it out. For example, it is legitimate to point out that an improvement in the quality of generative models could be used to generate deepfakes for disinformation. On the other hand, it is not needed to point out that a generic algorithm for optimizing neural networks could enable people to train models that generate Deepfakes faster.
		\item The authors should consider possible harms that could arise when the technology is being used as intended and functioning correctly, harms that could arise when the technology is being used as intended but gives incorrect results, and harms following from (intentional or unintentional) misuse of the technology.
		\item If there are negative societal impacts, the authors could also discuss possible mitigation strategies (e.g., gated release of models, providing defenses in addition to attacks, mechanisms for monitoring misuse, mechanisms to monitor how a system learns from feedback over time, improving the efficiency and accessibility of ML).
	\end{itemize}
	
	\item {\bf Safeguards}
	\item[] Question: Does the paper describe safeguards that have been put in place for responsible release of data or models that have a high risk for misuse (e.g., pretrained language models, image generators, or scraped datasets)?
	\item[] Answer: \answerNA{} 
	\item[] Justification: The paper poses no such risks.
	\item[] Guidelines:
	\begin{itemize}
		\item The answer NA means that the paper poses no such risks.
		\item Released models that have a high risk for misuse or dual-use should be released with necessary safeguards to allow for controlled use of the model, for example by requiring that users adhere to usage guidelines or restrictions to access the model or implementing safety filters. 
		\item Datasets that have been scraped from the Internet could pose safety risks. The authors should describe how they avoided releasing unsafe images.
		\item We recognize that providing effective safeguards is challenging, and many papers do not require this, but we encourage authors to take this into account and make a best faith effort.
	\end{itemize}
	
	\item {\bf Licenses for existing assets}
	\item[] Question: Are the creators or original owners of assets (e.g., code, data, models), used in the paper, properly credited and are the license and terms of use explicitly mentioned and properly respected?
	\item[] Answer: \answerNA{} 
	\item[] Justification: The paper does not use existing assets.
	\item[] Guidelines:
	\begin{itemize}
		\item The answer NA means that the paper does not use existing assets.
		\item The authors should cite the original paper that produced the code package or dataset.
		\item The authors should state which version of the asset is used and, if possible, include a URL.
		\item The name of the license (e.g., CC-BY 4.0) should be included for each asset.
		\item For scraped data from a particular source (e.g., website), the copyright and terms of service of that source should be provided.
		\item If assets are released, the license, copyright information, and terms of use in the package should be provided. For popular datasets, \url{paperswithcode.com/datasets} has curated licenses for some datasets. Their licensing guide can help determine the license of a dataset.
		\item For existing datasets that are re-packaged, both the original license and the license of the derived asset (if it has changed) should be provided.
		\item If this information is not available online, the authors are encouraged to reach out to the asset's creators.
	\end{itemize}
	
	\item {\bf New Assets}
	\item[] Question: Are new assets introduced in the paper well documented and is the documentation provided alongside the assets?
	\item[] Answer: \answerNA{} 
	\item[] Justification: The paper does not release new assets.
	\item[] Guidelines:
	\begin{itemize}
		\item The answer NA means that the paper does not release new assets.
		\item Researchers should communicate the details of the dataset/code/model as part of their submissions via structured templates. This includes details about training, license, limitations, etc. 
		\item The paper should discuss whether and how consent was obtained from people whose asset is used.
		\item At submission time, remember to anonymize your assets (if applicable). You can either create an anonymized URL or include an anonymized zip file.
	\end{itemize}
	
	\item {\bf Crowdsourcing and Research with Human Subjects}
	\item[] Question: For crowdsourcing experiments and research with human subjects, does the paper include the full text of instructions given to participants and screenshots, if applicable, as well as details about compensation (if any)? 
	\item[] Answer: \answerNA{} 
	\item[] Justification: The paper does not involve crowdsourcing nor research with human subjects.
	\item[] Guidelines:
	\begin{itemize}
		\item The answer NA means that the paper does not involve crowdsourcing nor research with human subjects.
		\item Including this information in the supplemental material is fine, but if the main contribution of the paper involves human subjects, then as much detail as possible should be included in the main paper. 
		\item According to the NeurIPS Code of Ethics, workers involved in data collection, curation, or other labor should be paid at least the minimum wage in the country of the data collector. 
	\end{itemize}
	
	\item {\bf Institutional Review Board (IRB) Approvals or Equivalent for Research with Human Subjects}
	\item[] Question: Does the paper describe potential risks incurred by study participants, whether such risks were disclosed to the subjects, and whether Institutional Review Board (IRB) approvals (or an equivalent approval/review based on the requirements of your country or institution) were obtained?
	\item[] Answer: \answerNA{} 
	\item[] Justification: The paper does not involve crowdsourcing nor research with human subjects.
	\item[] Guidelines:
	\begin{itemize}
		\item The answer NA means that the paper does not involve crowdsourcing nor research with human subjects.
		\item Depending on the country in which research is conducted, IRB approval (or equivalent) may be required for any human subjects research. If you obtained IRB approval, you should clearly state this in the paper. 
		\item We recognize that the procedures for this may vary significantly between institutions and locations, and we expect authors to adhere to the NeurIPS Code of Ethics and the guidelines for their institution. 
		\item For initial submissions, do not include any information that would break anonymity (if applicable), such as the institution conducting the review.
	\end{itemize}
	
\end{enumerate}

\end{document}